\pdfoutput=1 % arXiv: compile with pdfLaTeX
\documentclass[11pt]{article}

\usepackage[a4paper,textwidth=6in,textheight=9.2in,headsep=0.3in]{geometry}
\usepackage[T1]{fontenc}
\usepackage{lmodern}
\usepackage{microtype}
\usepackage[authoryear,round]{natbib}
\setcitestyle{authoryear,round,citesep={;},aysep={,},yysep={;}}
\usepackage{titlesec}
\titleformat*{\section}{\large\bfseries}
\titleformat*{\subsection}{\normalsize\bfseries}
\titleformat*{\subsubsection}{\normalsize\bfseries\itshape}
\titlespacing*{\section}{0pt}{2.2ex plus .5ex minus .2ex}{1.1ex plus .2ex}
\titlespacing*{\subsection}{0pt}{1.8ex plus .5ex minus .2ex}{0.8ex plus .2ex}
\usepackage{fancyhdr}
\renewcommand{\headrulewidth}{0pt}
\usepackage{caption}
\usepackage{amsmath,amsfonts,bm}

\def\eqref#1{equation~\ref{#1}}
\def\1{\bm{1}}

\DeclareMathAlphabet{\mathsfit}{\encodingdefault}{\sfdefault}{m}{sl}
\SetMathAlphabet{\mathsfit}{bold}{\encodingdefault}{\sfdefault}{bx}{n}

\usepackage[colorlinks,linkcolor=blue!55!black,citecolor=blue!55!black,urlcolor=blue!55!black]{hyperref}
\usepackage{booktabs}
\usepackage{xcolor}
\usepackage{graphicx}
\usepackage{amsmath,amssymb}
\usepackage{amsthm}
\usepackage{placeins}   % \FloatBarrier, used in the appendix
\usepackage{url}

\newtheorem{theorem}{Theorem}
\newtheorem{proposition}{Proposition}
\newtheorem{lemma}{Lemma}
\newtheorem{corollary}{Corollary}
\newtheorem{definition}{Definition}

\newcommand{\baselinev}{\textsc{Baseline}}
\newcommand{\notoken}{\textsc{NoTokenAttn}}
\newcommand{\mlponly}{\textsc{MlpOnly}}
\newcommand{\mlpwide}{\textsc{MlpOnly-Wide}}
\newcommand{\untied}{\textsc{Untied}}
\newcommand{\untiedln}{\textsc{Untied-LN}}
\newcommand{\frozensl}{\textsc{FrozenSlice}}
\newcommand{\sliceonce}{\textsc{SliceOnce}}
\newcommand{\flashslice}{FlashSlice}

\newcommand{\kcoupling}{$K$-\textsc{Coupling}}
\title{\bfseries Does Transolver really need a Transformer?}

\author{Shizheng Wen \qquad Siddhartha Mishra\\[4pt]
\normalsize Seminar for Applied Mathematics, ETH Zurich, Switzerland\\[1pt]
\normalsize\texttt{\{shizheng.wen, siddhartha.mishra\}@sam.math.ethz.ch}\\[6pt]
\normalsize Code: \url{https://github.com/Shizheng-Wen/flashslice}}
\date{}

\begin{document}
% only the appendix is listed in the table of contents
\addtocontents{toc}{\protect\setcounter{tocdepth}{-1}}

\maketitle
\thispagestyle{plain}

\begin{abstract}
The widely used Transolver family of neural operators is based on physics-attention, which softly assigns the points of an unstructured mesh to a small number of
slices, applies self-attention among the resulting tokens, and broadcasts
the result back to the points. We provide a comprehensive empirical and theoretical analysis to elucidate the mechanisms which are responsible for model performance. To this end, we perform careful ablations on a challenging suite of nine 3D fluid dynamics benchmarks to find that replacing token attention with a constant linear map does not affect the accuracy. Thus, Transolver does not need a Transformer at all. However, removing the global mixing (slicing/deslicing) leads to performance collapse, and how often it must be repeated grows with the nonlocality of the operator. We leverage the theory of averaging neural operators to describe what model expressivity can account for in these findings: just slicing/deslicing, in conjunction with pointwise MLPs, already suffices for universal approximation of continuous operators, so that attention is redundant for expressivity. Finally, we provide a novel FlashAttention-style efficient implementation of the key slicing/deslicing module of Transolver. This 
\flashslice{} kernel streams slice and deslice over the points without
materializing the heavy slice-weight tensor, while reproducing the best available implementation to floating point error. At the same time, it leads to very significant memory and compute savings, particularly at large slice counts. 

\end{abstract}

\section{Introduction}
\label{sec:intro}
Partial differential equations (PDEs) describe diverse phenomena across the
physical sciences and engineering \citep{evans2010pde}, and simulating them accurately is
a central task of computational science. Classical numerical methods \citep{quarteroni1994numerical} do so
reliably, but a high-fidelity simulation of a flow around an industrial
geometry can occupy an HPC cluster for hours to days. This
cost becomes prohibitive in the \emph{many-query} settings of design
optimization, uncertainty quantification and inverse problems, where the
same equation is solved for thousands of geometries and operating
conditions. Scientific machine learning \citep{NAMLbook}, in the specific form of \emph{operator learning}, addresses this bottleneck by learning the \emph{solution
operator} of a PDE from a dataset of simulations, so that a new instance
costs a very cheap single forward model pass.

The first generation of operator learning algorithms was designed for structured
grids: the Fourier neural operator \citep{li2021fourier}, DeepONet
\citep{lu2021learning}, convolutional neural operators
\citep{raonic2023cno} and operator vision transformers \citep{pos1} are based on Cartesian grids.
Industrial engineering simulations, however, are based on unstructured meshes and point
clouds that can resolve localized features such as boundary layers, shocks and vortices at widely varying
point densities. A growing literature targets this setting with graph
neural networks \citep{pfaff2021meshgraphnets,rigno}, geometry-informed hybrids
\citep{li2023gino,wen2025gaot} and transformers whose tokens are mesh
point features \citep{cao2021transformer,hao2023gnot,li2023transformer}. Among the
latter, Transolver \citep{wu2024transolver} has become a widely used choice.
Its \emph{physics-attention} layer softly assigns each of the $N$ mesh
points  to a small number $G$ of \emph{slices}, computes self-attention
among the $G$ resulting tokens, and broadcasts the result back to the
original points, at a cost linear in $N$. It has been extended by Transolver++ \citep{luo2025transolverpp},
Transolver-3 \citep{zhou2026transolver3} and GeoTransolver
\citep{adams2025geotransolver}, among others, and adopted in industrial frameworks such
as NVIDIA PhysicsNeMo.

Despite its popularity, we argue that the Transolver model itself has not been properly analyzed and understood. In particular, which of its three ingredient operations: slicing/deslicing, attention among the tokens and persistent pointwise MLPs, is responsible for the accuracy shown by the model is completely unclear. Although the focus of successive works on Transolver has been to modify the attention mechanism, it is not even clear whether \emph{Transolver really needs a Transformer inside it or not?} or whether memory-heavy operations such as slicing/deslicing contribute significantly to its accuracy. Our main motivation for this paper is to perform a comprehensive empirical and rigorous mathematical analysis of the Transolver model to precisely answer such questions. In this regard, our contributions in this paper are, 
\begin{figure}[t]
\centering
\includegraphics[width=\linewidth]{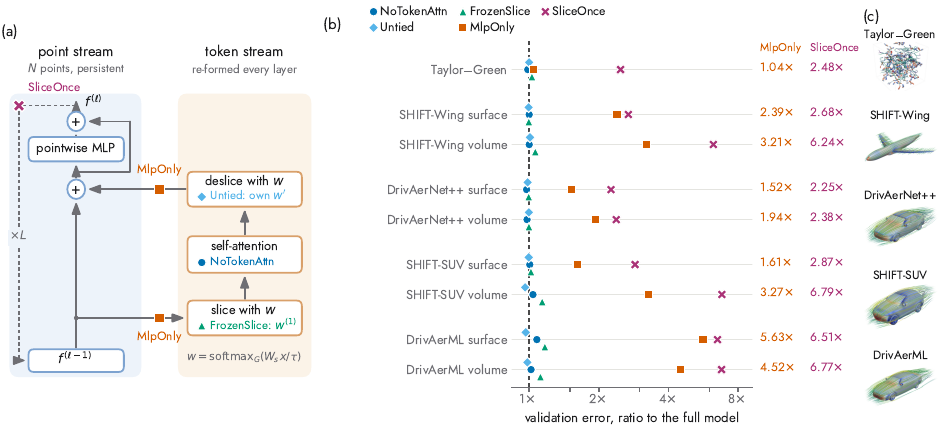}
\caption{\textbf{What Transolver's accuracy does, and does not, need.}
(a)~A Transolver layer: a persistent full-resolution point stream, coupled
through slice and deslice with shared per-point slice weights $w$ to
$H{\times}G$ tokens that are re-formed at every layer. Each ablation is
drawn on the element it changes; $w^{(1)}$ denotes the layer-$1$ weights
reused by \frozensl{} and $w'$ the separate deslice weights of \untied{}.
(b)~Validation-error ratio to the full model for the five ablations on
nine benchmarks, markers as in (a). The variants that keep the coupling
(\notoken{}, \untied{}, \frozensl{}) sit at $\approx1\times$; removing it
(\mlponly{}) costs $1.04$--$5.6\times$ and collapsing the point stream
(\sliceonce{}) up to $6.8\times$. (c)~One reference solution per benchmark
family, each beside the rows of (b) it belongs to: Taylor--Green
vorticity, and surface pressure with streamlines on SHIFT-Wing,
DrivAerNet++, SHIFT-SUV and DrivAerML.}
\label{fig:money}
\end{figure}
\begin{itemize}
\item \textbf{A careful empirical analysis of physics-attention across nine
benchmarks} (\S\ref{sec:anatomy}--\ref{sec:experiments}). We perform controlled
ablations of every design decision within the Transolver model, trained under an  identical protocol on nine
challenging 3D benchmarks (see Fig.~\ref{fig:money}) to demonstrate the \emph{surprising fact that replacing token
self-attention by a constant linear map matches the full model on all nine ($-2.6\%$ to $+7.9\%$)}, at every slice count from $8$ to $256$. Moreover, we identify that the slice-deslice operations are essential and removing them leads to significant loss of accuracy, increasing the error by $1.04$--$5.6\times$. Finally, slicing once and running a deep transformer on the tokens leads to errors $2.2$--$6.8\times$ those of the baseline model. Controlled ablations show that most of this penalty stems from the missing pointwise decoder, and that the number of coupling rounds a task needs grows with its nonlocality.
\item \textbf{A mathematical explanation from the theory of neural
operators} (\S\ref{sec:theory}). By interpreting slicing/deslicing as finite-rank nonlocal integral operators with a learned
partition of unity and leveraging the universal approximation theorem of
\citet{lanthaler2023nonlocality}, we show that
the transformer-free Transolver is already universal, that a purely pointwise stack is
bounded away from every nonlocal target at any width and depth, and that
the nonlocal information a stack can carry grows with its number of coupling rounds.
\item \textbf{\flashslice{}} (\S\ref{sec:flashslice}). We implement the key slicing/deslicing module as fused Triton kernels that stream over the points and never write
the slice-weight tensor, with a single-tile path for small slice counts
and a $G$-blocked path with saved softmax statistics for arbitrary ones.
They are exact and bit-deterministic, and train $1.35$--$1.7\times$ faster in up to $25\%$
less memory at the standard configuration over the current implementation and $4$--$6.5\times$ faster in
$85$--$96\%$ less memory at large slice counts.
\end{itemize}
Thus, we provide a comprehensive empirical analysis that identifies the key causal mechanism underlying Transolver's accuracy, explain it with mathematical theory and provide a highly efficient implementation of this core layer. 
\section{Physics-Attention as a Low-Rank Nonlocal Operator}
\label{sec:anatomy}
\paragraph{Problem formulation.} We consider PDEs $\mathcal{L}(u;a)=0$ in $\Omega\subset\mathbb{R}^{d}$
with boundary conditions $\mathcal{B}(u;a)=0$ on $\partial\Omega$, where
$a\in X$ collects the geometry, coefficients, initial data and operating
conditions of an instance and $u\in Y$ is its solution and $X,Y$ are suitable function spaces. The solution
operator $\mathcal{S}:X\to Y$, $a\mapsto u$, is what a neural operator
$\Psi_\theta$ approximates, given training pairs $(a_k,\mathcal{S}(a_k))$
with $a_k$ drawn from a distribution $\mu$ on $X$. In practice $a$ and $u$
are given on a mesh $\{x_i\}_{i=1}^{N}\subset\Omega$, so $\Psi_\theta$ maps
per-point inputs to per-point outputs. 
\paragraph{The Transolver layer.} Transolver \citep{wu2024transolver}
lifts the per-point inputs to features $f_i^{(\ell)}\in\mathbb{R}^{c}$ and
applies $L$ layers, with $1 \leq \ell \leq L$, of two interleaved streams. The \emph{token stream} is
physics-attention: per head, a linear projection of the features with a
learned temperature $\tau$ produces \emph{slice weights}
$w_{i,g}=\mathrm{softmax}_{g}\big((W_{s}f_i+b)/\tau\big)$, a soft
assignment of each point to $G$ slices (with layer index $\ell$ suppressed for notational convenience); the \emph{slice} operator outputs tokens as
weighted averages $z_g=\sum_i w_{i,g}v_i/\sum_i w_{i,g}$ of projected
features $v_i$; self-attention is applied among the $G$ tokens; and
\emph{deslice} broadcasts the result back, $u_i=\sum_g w_{i,g}z'_g$,
added residually to the features. The \emph{point stream} is a pointwise
two-layer MLP with residual connection at full resolution after every
attention sublayer, forming a persistent point-level representation. Two properties are worth mentioning: the slice weights are recomputed from the
current features at every layer, and the same weights are used in the slice and
deslice operators, i.e.\ they are \emph{tied}.

\paragraph{A rank-$G$ integral operator.} In the continuum limit, the slice
weights define $G$ functions $s_g:\Omega\to[0,1]$, $w_{i,g}=s_g(x_i)$, with $\sum_g s_g\equiv1$,
forming a learnable, data-dependent partition of unity. Informally, one can think of slice as analysis against
this set of functions and deslice as synthesis in its span, so that one
physics-attention sublayer implements the rank-$G$ integral operator
\begin{equation}
(\mathcal{K}v)(x) \;=\; \sum_{g,h=1}^{G} s_g(x)\, A_{gh}\,
\frac{\int_\Omega s_h(y)\, v(y)\,\mathrm{d}\mu(y)}{\int_\Omega s_h(y)\,
\mathrm{d}\mu(y)}.
%K(x,y) \;=\; \sum_{g,h} s_g(x)\,A_{gh}\,\tilde{s}_h(y).
\label{eq:lowrank}
\end{equation}
Here the core $A\in\mathbb{R}^{G\times G}$ is the nonlinear map produced by the token
self-attention in \baselinev{} version of Transolver, and is therefore content-dependent,
whereas in the attention-free variant \notoken{} below the core is constant:
$A$ is the identity, so the tokens do not interact, and only a learned linear
map acts on the channels of each token; a full layer reads
$\bar f^{(\ell)}=f^{(\ell-1)}+\mathcal{K}_\ell f^{(\ell-1)}$,
$f^{(\ell)}=\bar f^{(\ell)}+\mathrm{MLP}\big(\bar f^{(\ell)}\big)$.
Equation~(\ref{eq:lowrank}) is clearly a low-rank instance of the nonlocal neural
operator layer of \citet{kovachki2023neural}, distinguished by the partition of unity $\{s^\ell_g\}$,
re-estimated from $f^{(\ell-1)}$ at every layer and by the interleaving
with a persistent pointwise stream. Its range is
$\mathrm{span}\{s^\ell_g\}$ regardless of $A$: thus, the attention re-parameterizes
the core but cannot enlarge what the slice/deslice pair, the
\emph{coupling}, expresses.

\paragraph{Ablation suite.} We isolate each design decision in Transolver with
variants of \baselinev{}. Each ablation changes exactly one ingredient in the full Transolver model,
except \untiedln{}, which is deliberately the conjunction of two ingredients
plus a normalization and reproduces the LinearNO layer, recently considered in 
\citet{hu2026linearno} (Appendix~\ref{app:linearno}).
\begin{center}\footnotesize
\setlength{\tabcolsep}{3pt}
\begin{tabular}{@{}l p{0.79\linewidth}@{}}
\toprule
\notoken{} & token self-attention $\to$ per-token linear map (no $QK^{\top}$ softmax; tokens do not interact)\\
\mlponly{} & attention sublayer removed entirely (pointwise residual MLP stack; no cross-point path)\\
\mlpwide{} & \mlponly{} widened to $112\%$ of \baselinev{} parameters (capacity control)\\
\untied{}  & deslice gets its own projection and temperature (tests weight tying)\\
\untiedln{} & \untied{} $+$ \notoken{} $+$ analysis-side normalization along $N$ (LinearNO)\\
\frozensl{} & slice weights from layer $1$, reused by all layers (tests per-layer re-slicing)\\
\sliceonce{} & slice once $\to$ transformer on $G$ tokens $\to$ deslice once (Perceiver; no point stream, affine read-out)\\
\sliceonce{}$+$\textsc{Dec} & \sliceonce{} followed by four pointwise MLP blocks (tests the missing decoder)\\
\kcoupling{} & attention sublayer in only $K$ of the $L$ layers, pointwise MLP in all (tests coupling frequency)\\
\bottomrule
\end{tabular}
\end{center}
\section{Experiments}
\label{sec:experiments}
\paragraph{Benchmarks.} We use nine dataset configurations from CFD, the three-dimensional Taylor--Green vortex which is the prototype of turbulent time-varying flows \citep{gencfd} and aerodynamics at industrial scale, namely DrivAerNet++
\citep{elrefaie2024drivaernet} surface and volume (fixed inflow),
SHIFT-SUV surface and volume (fixed inflow, geometry the only input),
the transonic SHIFT-Wing surface and volume (varying inflow and geometry),
and DrivAerML \citep{ashton2024drivaerml} surface and volume (fixed inflow;
up to $1.3\times10^{8}$ points per sample). The suite spans a very diverse set of operators, ranging from inputs that
contain the local field state to solutions set by a global interaction of
geometry and operating conditions (Appendix~\ref{app:datasets}).
\paragraph{Protocol.} For each benchmark, all variants of Transolver are trained to a common step budget on identical data pipelines with the Transolver backbone of
\citet{wu2024transolver} ($L{=}8$, width $256$, $8$ heads, $G{=}32$;
$3.87$M parameters); the headline pair \baselinev{}/\notoken{} uses three
seeds. We report the best validation relative $L^{1}$ error for each variant. Step budgets are checked rather than assumed  and every ablation is verified from the
trained artifacts through checkpoint contents and bit-level freeze tests, see Appendix~\ref{app:protocol} for details. 
\paragraph{The token self-attention is removable on all nine
benchmarks, i.e., Transolver does not really need a Transformer!} We observe from Table~\ref{tab:main} that \notoken{}, which
replaces the content-dependent attention among the $G$ tokens by a
learned linear map applied to each token separately, so that the tokens no
longer interact, matches \baselinev{} everywhere: it
lies within $-2.6\%$ to $+1.8\%$ on seven benchmarks and at $+3.8\%$ and
$+7.9\%$ on the remaining two (SHIFT-SUV volume and DrivAerML surface),
against the $1.04$--$6.8\times$ spread of the coupling ablations in the
same table, and on four benchmarks it is the best variant overall (tied
with \baselinev{} on SHIFT-Wing volume). 
\paragraph{Untying the projections is not what makes it removable.} A
possible explanation of the above result is that tying the slice and
deslice projections confines feature interaction to the same token, so
that the attention has to supply the missing cross-slice routing. On this reading, untying
these projections is what renders attention redundant \citep{hu2026linearno}. This explanation and ours (\S\ref{sec:theory})
disagree on the tied and attention-free
layer, which \citet{hu2026linearno} predict to perform poorly, whereas \notoken{} already
demonstrates that this is not the case. Moreover, in Fig.~\ref{fig:controls}(a), we present results for the $2\times2$ grid
of token attention $\times$ projection tying on SHIFT-SUV surface, the
untied and attention-free cell corresponds exactly to the LinearNO layer of \citet{hu2026linearno} ported
into our scaffold at matched latent width and capacity
(Appendix~\ref{app:linearno}). The grid clearly shows that there is no difference between these configurations: all four cells lie
within $1.05\%$ of one another, and both main effects ($+0.54\%$ for
removing the attention, $-0.51\%$ for untying) and their interaction
($-0.41\%$) are about one standard deviation of the baseline's own
three-seed spread of $\pm0.47\%$. The attention is redundant, and the
tying is clearly orthogonal to that redundancy. A second benchmark, SHIFT-Wing surface, confirms this (all cells within $0.7\%$, Table~\ref{tab:rb-factorial}). LinearNO also normalizes its
slice weights by a softmax over all points instead of per point, and at its
published latent width a training iteration is $6.8\times$ slower with no
change in accuracy (Appendix~\ref{app:linearno}).
\paragraph{The result does not depend on the slice count.} The core $A$
is $G\times G$, so a larger number of slices $G$ provides attention more information to mix.
In Fig.~\ref{fig:controls}(b), we present results with $G$ swept over a $32\times$ range at a
matched budget. \notoken{} stays within $1.5\%$ of \baselinev{} at every
$G$ and within $1\%$ for any $G\ge32$, the gap tightening as $G$ grows,
including at $G{=}256$, where the tokens per head outnumber the head
dimension eightfold. Moreover, model capacity is always matched
along the sweep ($G{:}\,32\to256$ adds $1.6\%$ to the parameter count). Notably, the
baseline itself moves by only $0.6\%$ over the sweep on this benchmark.
\paragraph{The token mixer need not be content-dependent.} On
DrivAerML surface, where removing the attention costs the most, a learned constant
$G{\times}G$ core ($0.984\times$ the error of \baselinev{}), a token-mixing MLP
($0.987\times$) and a state-space scan \citep{tiwari2025lamo} ($1.017\times$) in
place of the attention all match it (Table~\ref{tab:rb-mixers}).
% ============================================================================
% T1 -- Main results table: "Does Transolver need a Transformer?"
% Numbers FINAL as of 2026-08-13. Provenance: cameleon repo, wandb project
% `transolver-attention`;
% per-run best eval read from trainer_state.json (archived on capstor store).
% SHIFT-Wing surface row added 2026-08-06 (round 1 + headline-pair seeds 43/44;
% every run completed the full 250k cosine except untied, see app:budgets).
%
% TEAM-7 and TEAM-13 REMOVED 2026-08-19: the advisor is not releasing those two
% datasets yet. The rows, their footnote, and everything else that rested on them
% are preserved verbatim in tables/ARCHIVED_electromagnetics.tex, which is
% deliberately not \input by the paper.
%
% SHIFT-SUV surface and volume rows added 2026-08-18/19 (50k protocol).
% SHIFT-Wing volume row added 2026-09-11: reworked files, one 12h leg, iso-step
% at 140k (scripts/collect_isostep.py); three seeds each for the headline pair,
% s42 of the baseline resumed once after a node loss (see app:budgets).
% The step footnotes b-g moved to Appendix C.1 (app:budgets) on 2026-09-23.
%
% Requires in preamble:
%   \usepackage{booktabs}
%   \usepackage{xcolor}
% ============================================================================
\begin{table}[t]
\caption{\textbf{Ablating physics-attention across nine benchmarks.} Best
validation relative $L^{1}$ (lower is better), averaged over the predicted
quantities in normalized units; per-quantity errors in physical units are in
Tables~\ref{tab:pervar-l1} and~\ref{tab:pervar-l2}. Mean$\pm$std over
three seeds for \baselinev{} and \notoken{}, single seed otherwise; best per
row in \textbf{bold}. Three significant figures; standard
deviations to one.}
\label{tab:main}
\begin{center}
\footnotesize
\setlength{\tabcolsep}{4pt}
\resizebox{\textwidth}{!}{%
\begin{tabular}{lcccccc}
\toprule
Dataset & \baselinev{} & \notoken{} & \mlponly{} & \untied{} & \frozensl{} &
 \sliceonce{} \\
\midrule
Taylor--Green & $0.0756{\scriptstyle\pm.0001}$ & $\mathbf{0.0745}{\scriptstyle\pm.0001}$ &
  $0.0786$ & $0.0755$ & $0.0777$ & $0.188$ \\
\midrule
\multicolumn{7}{l}{\emph{SHIFT-Wing}}\\
~~surface & $0.0580{\scriptstyle\pm.0005}$ & $0.0579{\scriptstyle\pm.0005}$ &
  $0.139$ & $\mathbf{0.0576}$ & $0.0577$ & $0.155$ \\
~~volume  & $\mathbf{0.127}{\scriptstyle\pm.0007}$ & $\mathbf{0.127}{\scriptstyle\pm.0008}$ &
  $0.408$ & $0.128$ & $0.135$ & $0.793$ \\
\midrule
\multicolumn{7}{l}{\emph{DrivAerNet++}}\\
~~surface & $0.193{\scriptstyle\pm.005}$ & $\mathbf{0.188}{\scriptstyle\pm.002}$ &
  $0.294$ & $0.190$ & $0.192$ & $0.435$ \\
~~volume  & $0.163{\scriptstyle\pm.0003}$ & $\mathbf{0.159}{\scriptstyle\pm.001}$ &
  $0.315$ & $0.162$ & $0.162$ & $0.386$ \\
\midrule
\multicolumn{7}{l}{\emph{SHIFT-SUV}}\\
~~surface & $0.157{\scriptstyle\pm.0007}$ & $0.157{\scriptstyle\pm.00003}$ &
  $0.252$ & $\mathbf{0.156}$ & $0.159$ & $0.449$ \\
~~volume  & $0.0601{\scriptstyle\pm.0006}$ & $0.0624{\scriptstyle\pm.0009}$ &
  $0.197$ & $\mathbf{0.0577}$ & $0.0683$ & $0.408$ \\
\midrule
\multicolumn{7}{l}{\emph{DrivAerML}}\\
~~surface & $0.089{\scriptstyle\pm.001}$ & $0.096{\scriptstyle\pm.000}$ &
  $0.501$ & $\mathbf{0.086}$ & $0.104$ & $0.579$ \\
~~volume  & $0.110{\scriptstyle\pm.005}$ & $0.112{\scriptstyle\pm.004}$ &
  $0.497$ & $\mathbf{0.108}$ & $0.123$ & $0.745$ \\
\bottomrule
\end{tabular}}
\end{center}
\end{table}

\paragraph{The slice/deslice coupling and the point stream are not
removable.} Removing the entire attention sublayer (\mlponly{}) costs
$1.04\times$ to $5.6\times$ (Table~\ref{tab:main}). This is not explainable in terms of model capacity: widening the pointwise MLP to $112\%$ of the baseline's
parameters (\mlpwide{}, Fig.~\ref{fig:controls}(c)) recovers at most
$3\%$ of the gap and leaves the model $1.6\times$ (surface) and
$1.9\times$ (volume) less accurate on DrivAerNet++. Nor is this difference attributable to budget differences; on SHIFT-SUV the \mlponly{} ratio grows from $1.52\times$ to $1.62\times$ (surface) and from
$2.43\times$ to $3.24\times$ (volume) as the budget rises from $10$k to
$50$k steps (Appendix~\ref{app:protocol}). Conversely, \sliceonce{} keeps
the attention and has \emph{more} parameters than the baseline, but
slices once and confines all further depth to the $G$ tokens; it is the
worst variant on every benchmark, at $2.2$--$6.8\times$. On the other hand, re-using the layer-$1$ slice weights in every layer (\frozensl{}) costs
only $-0.5\%$ to $+17\%$, mostly on the volume tasks. Thus, cross-point coupling
through slice/deslice, re-estimated at every layer and interleaved at full
resolution with the pointwise stream, is clearly the most performant mechanism.
\paragraph{How often must the coupling be repeated?} \sliceonce{}
couples only once, but it also drops the point stream and the pointwise decoder. To separate the two, we keep the pointwise MLP in all layers and the
attention sublayer in only $K$ of them (\kcoupling{}), and we add four pointwise
MLP blocks to \sliceonce{} (Table~\ref{tab:rb-coupling}). First, most of the
\sliceonce{} penalty stems from the missing decoder: adding it reduces the
penalty from $2.87\times$ to $1.28\times$ (SHIFT-SUV surface) and from
$6.31\times$ to $3.02\times$ (DrivAerML surface). Second, the number of rounds a
task needs follows its nonlocality: one round nearly suffices on SHIFT-SUV
surface ($1.02\times$), whereas the error on DrivAerML surface falls with every
added round ($1.38\times$, $1.12\times$ and $1.02\times$ at $K=1$, $2$ and $4$).
\paragraph{The coupling is needed where the physics is nonlocal.} The \mlponly{} penalty tracks the physics of each
task. It is negligible ($1.04\times$) on Taylor--Green, whose inputs
contain the current field state at every point so that short-horizon
prediction is largely local. On the aerodynamic benchmarks the field at a
point is set by a global interaction between the whole body and the flow
around it, which a function of local descriptors can absorb only in part
(Appendix~\ref{app:mechanism}), and the size of the penalty does not follow
the inflow: it is moderate ($1.5$--$3.3\times$) for DrivAerNet++ and
SHIFT-SUV and largest for DrivAerML (up to $5.6\times$), all three at fixed
inflow, while SHIFT-Wing, the only benchmark whose inflow varies, lies
between these ($2.4$--$3.2\times$). SHIFT-SUV isolates a second axis, the
kind of target: same cars, split, budget and fixed condition, differing
only in whether the target is a trace on the body or a field transported
away from it, and every coupling-sensitive number roughly doubles from
surface to volume (\mlponly{} $1.61\to3.27\times$, \sliceonce{}
$2.87\to6.79\times$); SHIFT-Wing shows the same changes under varying
inflow (\mlponly{} $2.39\to3.21\times$, \sliceonce{} $2.68\to6.24\times$). The
per-field decomposition makes the same point inside each task
(Figure~\ref{fig:controls}(d)): on every aerodynamic benchmark the
pressure, which can be recovered as the solution of a global elliptic problem, degrades more than
the velocity or wall shear beside it when the coupling is removed, by
$1.8$--$14.8\times$ against $1.3$--$5.0\times$. Figure~\ref{fig:fields}
(Appendix~\ref{app:results}) shows the corresponding fields, and
Appendix~\ref{app:fieldset} the same figure on every benchmark.
\begin{figure}[h!]
\centering
\includegraphics[width=\linewidth]{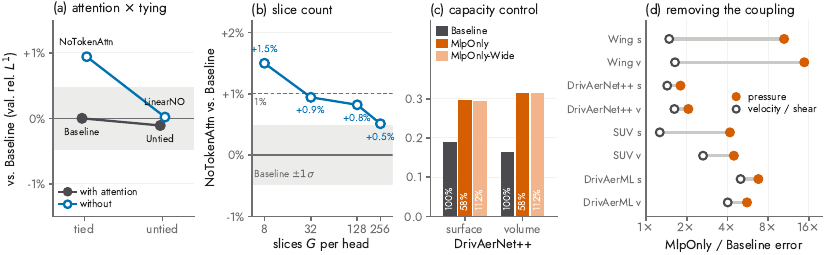}
\caption{\textbf{Four controls.} (a)~Token attention $\times$ projection
tying on SHIFT-SUV surface, change against the tied \baselinev{};
\untiedln{} is the released LinearNO layer \citep{hu2026linearno} in our
scaffold. (b)~\notoken{} against \baselinev{} across the slice count $G$,
same benchmark. Grey band in (a, b): the baseline's three-seed spread.
(c)~Capacity control on DrivAerNet++; bars labelled by parameter count
relative to \baselinev{}. (d)~Test error of \mlponly{} over \baselinev{}
per field, pressure against velocity or wall shear
(Table~\ref{tab:pervar-l1}).}
\label{fig:controls}
\end{figure}
\section{Why the Attention Is Removable and the Coupling Is Not}
\label{sec:theory}
Our aim in this section is to delimit what model expressivity can and cannot account for in the results of Sec.~\ref{sec:experiments} (Table~\ref{tab:budgets}). Our starting point is the following proposition on the universality of the Transolver layer, Eq.~(\ref{eq:lowrank}),
\begin{proposition}[The constant-core class is universal]
\label{prop:p1}
Let $\Omega\subset\mathbb{R}^d$ be a bounded Lipschitz domain and
$\Psi^\dagger:\mathcal{X}\to\mathcal{Y}$ a continuous operator between
$C^{s}$ or $W^{s,p}$ spaces of functions on $\Omega$. Fix $G\ge1$ and let $\Psi_A$ be the class of operators realized by a
stack of layers of the form~(\ref{eq:lowrank}) with a constant core $A$
(\notoken{}), interleaved with pointwise MLPs of unbounded width. Then, for every compact $\mathsf{K}\subset\mathcal{X}$ and $\epsilon>0$ there exists an operator $\Psi\in\Psi_A$ satisfying
$\sup_{u\in\mathsf{K}}\|\Psi(u)-\Psi^\dagger(u)\|_{\mathcal{Y}}\le\epsilon$.
\end{proposition}
The proof of this proposition is detailed in Appendix~\ref{app:theory} and is based on the universality of \emph{averaging neural operators} (ANOs) of \cite{lanthaler2023nonlocality}, where the authors show that neural operators of the form, 
\begin{equation}
\Psi(u)(x) \;=\; Q\Big(\sigma\big(\textstyle\int_\Omega R(u(y),y)\,\mathrm{d}y\big),\,x\Big),
\label{eq:ano}
\end{equation}
with pointwise MLPs $R,Q$ already suffice for universal approximation of continuous operators. Setting all slice logits to zero makes $s_g\equiv1/G$ a partition of
unity, every token becomes the global mean, and Eq.~(\ref{eq:lowrank}) provides the ANO
coupling \eqref{eq:ano}, while the pointwise MLPs supply $R$ and $Q$, and the construction uses
no explicit property of $G$.

The proposition leads to the following immediate conclusions i) to the best of the authors' knowledge, it provides the first universality result for the original Transolver architecture as it is straightforward to realize token attention weights to obtain a constant $A$ in \eqref{eq:lowrank}, making Transolver universal and ii) the result clearly suggests that using a transformer is not necessary for universal approximation, so that there is no expressivity-based reason for the attention to be needed. The proposition neither separates \baselinev{} from \notoken{} (both are universal) nor predicts trained accuracy; it predicts that a constant core loses nothing, which Table~\ref{tab:rb-mixers} confirms.

We can also explain why pointwise MLPs do not suffice for accuracy as reported in Sec.~\ref{sec:experiments}. To this end, we define
\begin{definition}[Locality gap]
\label{def:localitygap}
Call $\Phi$ pointwise if $\Phi(u)(x)$ depends on $u$ only through $u(x)$.
For a continuous operator $\Psi^\dagger:\mathcal{X}\to\mathcal{Y}$ and
compact $\mathsf{K}\subset\mathcal{X}$ let
$\delta_{\mathrm{loc}}(\Psi^\dagger):=\inf_{\Phi\ \mathrm{pointwise}}\
\sup_{u\in\mathsf{K}}\|\Psi^\dagger(u)-\Phi(u)\|_{\mathcal{Y}}$.
\end{definition}
\begin{corollary}[Depth does not repair \mlponly{}]
\label{cor:mlponly-depth}
A composition of pointwise maps is pointwise, and the pointwise class is
closed under uniform limits on $\mathsf{K}$. Hence for every $\Psi^\dagger$
with $\delta_{\mathrm{loc}}(\Psi^\dagger)>0$, the error of \mlponly{} is
at least $\delta_{\mathrm{loc}}(\Psi^\dagger)$ at every depth and every
width.
\end{corollary}
Thus, \mlponly{} cannot be universal as long as the operator is not entirely pointwise. Moreover, this is the sense in which the \mlponly{} gap is not a capacity deficit
that \mlpwide{} could close. It also orders the tasks:
$\delta_{\mathrm{loc}}$ is largest for operators with
global support, the elliptic pressure problem of low-Mach aerodynamics
being the canonical case, and smallest for the short-horizon map of Taylor--Green.
The corollary is elementary, but testable: on SHIFT-SUV surface the
\mlponly{} gap is $1.58\times$, $1.61\times$ and $1.56\times$ at $L=4$, $8$ and
$16$ (Table~\ref{tab:rb-depth}).

Finally, the following proposition is suggestive of some of the other observations of Sec.~\ref{sec:experiments},
\begin{proposition}[Nonlocal budgets]
\label{prop:p3}
Let $c$ be the width of the point stream. (i)~If $\Psi^\dagger$ is $\epsilon$-approximable by an ANO \eqref{eq:ano} of encoding
width $J\le c-k$, with $k$ channels carrying the coordinates, then a
constant-core stack of layers of the form~(\ref{eq:lowrank}) attains accuracy
$\epsilon$ for \emph{every} $G\ge1$. (ii)~One application of
$\mathcal{K}$ in \eqref{eq:lowrank} transmits at most $G\,c$ real functionals of the field, and
every point receives a fixed linear image of them. A stack that couples at
each of $L$ layers carries at most $L\,G\,c$; \sliceonce{} carries at most
$G\,c$ whatever the depth of its token stack; \mlponly{} carries $0$.
\end{proposition}
Part~(i) of the above proposition suggests why accuracy may be flat over a $32\times$ range of $G$
(Figure~\ref{fig:controls}(b)): universality is not bought with the slice
count. Part~(ii) provides the asymmetry behind \sliceonce{}: all nonlocal
information reaching its output passes through one round of $G\,c$
numbers, whereas variants that couple at every layer have $L$ rounds,
each re-estimated from transformed features, and only these improve with
depth. The bound is rather loose ($G\,c\approx8\times10^{3}$ here); the size of the
penalty comes from a second missing ingredient, the universal pointwise
decoder $Q$ of \eqref{eq:ano}, which \sliceonce{} replaces by an
affine readout in the tokens with coefficients $s_g(x)$.
In line with this, supplying the decoder removes most of the \sliceonce{}
penalty, and the error decreases with the coupling rounds where the locality gap
is large (Sec.~\ref{sec:experiments}). Being an upper bound, part~(ii) does not
determine the magnitudes (Table~\ref{tab:budgets}).
\begin{table}[t]
\caption{\textbf{What the theory settles, and what it leaves to
experiment.} Nonlocal budget and hard ceiling per variant, against the
measured cost relative to \baselinev{} (Tables~\ref{tab:main},
\ref{tab:rb-coupling} and~\ref{tab:rb-depth}).}
\label{tab:budgets}
\begin{center}\footnotesize
\setlength{\tabcolsep}{4pt}
\begin{tabular}{llll}
\toprule
variant & nonlocal budget & hard ceiling? & measured cost \\
\midrule
\baselinev{} / \untied{} & $L\,G\,c$ & no (Prop.~\ref{prop:p1}) & reference \\
\notoken{}   & $L\,G\,c$ & no (Prop.~\ref{prop:p1})            & $-2.6\%$ to $+7.9\%$ \\
\frozensl{}  & $L\,G\,c$ & no (\S\ref{app:theory-p3})          & $-0.5\%$ to $+17\%$ \\
\kcoupling{} & $K\,G\,c$ & no (Prop.~\ref{prop:p1}) & $1.01$--$1.38\times$ ($K{=}1,2,4$) \\
\sliceonce{} & $G\,c$    & yes, rank (Prop.~\ref{prop:p3})     & $2.2$--$6.8\times$; $1.28$--$3.0\times$ with a decoder \\
\mlponly{}   & $0$       & yes, locality (Cor.~\ref{cor:mlponly-depth}) & $1.04$--$5.6\times$; $1.56$--$1.61\times$ at $L{=}4$--$16$ \\
\bottomrule
\end{tabular}
\end{center}
\end{table}
\section{\flashslice{}: The Reduced Layer as a Kernel}
\label{sec:flashslice}

\paragraph{The bottleneck.} After removing attention, a physics-attention
sublayer is simply a slice $\to$ per-token linear map $\to$ deslice operation, and the coupling is what dominates compute and memory. The currently available eager implementation of Transolver forms
the slice weights $w\in\mathbb{R}^{B\times H\times N\times G}$ explicitly
and keeps them for the backward pass. For $N$ in the millions, this is the dominant activation of the layer and
the only per-layer term that also grows with $G$, so slice count and depth
compete for the same memory budget (Figure~\ref{fig:systems}(b,~c)). Slice and deslice cost
$O(NGD)$ multiply-adds for head width $D$, comparable to the pointwise
MLPs at $G{=}32$ and most of the layer at $G{=}1024$. Hence, our aim in this section is to describe \flashslice{}, a much more efficient implementation of the slicing kernel. 

\paragraph{Design.} \flashslice{} never writes $w$. Each pass streams
over the points in tiles, forms the membership tile in registers, applies
it and moves on, and the backward pass recomputes the tile from what was
saved: the trade FlashAttention \citep{dao2022flashattention} makes for
softmax attention, applied to a pooling that is not attention. The query
is a point, the keys are the $G$ slot projections, the softmax runs over
the slots, and the token state between the passes is $H\times G\times D$
floats ($32$\,KB at our configuration). Two things differ from attention: slice is the \emph{transpose} of
deslice, accumulating per slot over points rather than per point over
slots, so each point's normalizer must be known before its contribution is
added; and the temperature $\tau$ is learned, its gradient cancelling row
by row in exact arithmetic, which makes it the most rounding-sensitive
quantity in the layer. Two kernel families sit behind
one switch and are chosen by shape.
\begin{itemize}
\item \emph{Single-tile kernels}, for $G$ and $D$ powers of two in
$[16,128]$, hold the whole slot axis in one tile, so that the softmax over
$G$ never leaves registers and no statistics are stored. Tiles are tuned
per $G$ (carrying the $G{=}32$ table to $G{=}128$ costs up to
$40$--$80\times$ through register spilling); they serve every shape of
our accuracy experiments.
\item \emph{$G$-blocked kernels}, for any $G\ge1$ and head or value widths up to $256$,
take the approach of FlashAttention's \emph{backward}. One online pass
forms the per-point softmax statistics $(m_n,l_n)$, two floats per point
and head, i.e.\ $2/G$ of $w$, and every kernel then recomputes
$w_{ng}=\exp(\mathrm{logit}_{ng}-m_n)/l_n$ one $G$-block at a time. A
deslice with no statistics in hand runs FlashAttention's forward instead,
an online softmax over $G$, and hands its statistics to a tied slice;
slice cannot run that loop, since its accumulators live on slots and mix
many points, so its programs own a $G$-block, stream over $N$ with the
finished statistics, and write partials that the host reduces once. The price is the statistics pass and one more
logits recompute per backward, so routing prefers the single-tile kernels
wherever they apply.
\end{itemize}
Neither family uses atomics, so two runs are bitwise identical, and both
are registered as PyTorch custom operators \citep{ansel2024pytorch2},
which makes the model one compiled graph instead of $34$ and lets every
ablation variant of this paper run through the same kernels. Shapes
outside both families ($D>256$) fall back to the eager path with a
warning recorded on the model, so that a flag left on cannot silently
train a baseline replica.

\paragraph{Relation to Transolver-3.} \citet{zhou2026transolver3} also avoid
the full $N\times G$ weight tensor, by tiling the points in PyTorch and
recomputing each tile under gradient checkpointing; each tile of $w$ is
still written to HBM, and the backward reruns the tile's whole forward.
\flashslice{} does the same inside the kernel, where the tile never leaves
registers. Their second change, moving the input and output projections
to token space by associativity, is orthogonal to ours and composes with
it. The two approaches compose: the tiling lowers the peak memory at the price
of time, whereas the fused kernels lower both and can run inside each tile, where
they keep the memory of the tiling and are $2.5\times$ faster
(Table~\ref{tab:rb-tiling}, Appendix~\ref{app:tiling}).

\paragraph{Precision.} We benchmark \flashslice{} under both fp32 and bf16
autocast (Appendix~\ref{app:systems}). Triton's default dot uses an fp32
FMA path, which becomes a compute bottleneck as $G$ grows. We therefore
provide several dot-precision levels, all with fp32 accumulation
(Table~\ref{tab:dotmodes}). The default uses fp32 dots; under bf16
autocast, tensor cores can be used for the value products and, at the
\texttt{bf16} level, for the logits as well. At the other levels the
logits remain in fp32 because the softmax Jacobian amplifies errors in
the slice weights. In a $50$k-step large-$G$ training run, bf16 dots
tracked a tf32 control throughout training while reducing step time by
$3.5\times$ (Appendix~\ref{app:systems}).

\paragraph{Exactness.} \flashslice{} is a drop-in, not an approximation: against an fp64
reference the fused error is at or below eager's on every output and every
parameter gradient, for every ablation variant and both families, under
bf16 autocast it matches the numerical class of the baseline's own
tensor-core matmuls, and a $60$-step optimizer trajectory from a shared
initialization stays within $2\times10^{-6}$ relative of eager. No
accuracy result of this paper changes when the kernel is on
(Appendix~\ref{app:systems}).

\begin{figure}[t]
\centering
\includegraphics[width=\linewidth]{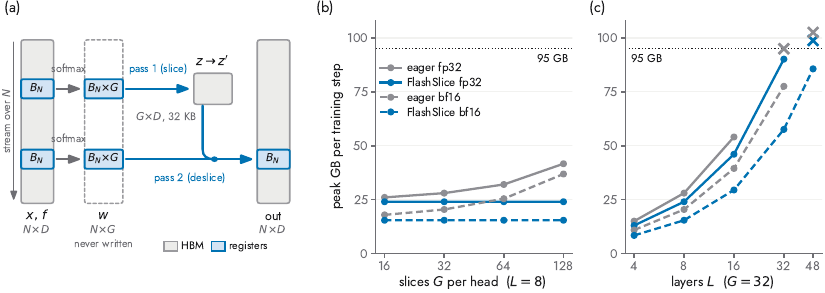}
\caption{\textbf{\flashslice{} turns the reduced layer into a systems
object.} (a)~The point tensors and the $32$\,KB token block live in HBM;
the slice weights exist only as one $B_N{\times}G$ tile in registers,
formed on the fly in each of two passes over the points and recomputed in
the backward, so the $N{\times}H{\times}G$ tensor is never written.
(b,~c)~Peak memory of a training step against slice count and depth
($N{=}262$k, eager PyTorch, no compilation or checkpointing); $\times$
marks the first configuration that does not fit in $95$\,GB. Step times of
the same sweeps are in Figure~\ref{fig:systems-time}
(Appendix~\ref{app:systems}).}
\label{fig:systems}
\end{figure}

\paragraph{\flashslice{} is faster and smaller, and the advantage grows
with the slice count.} Table~\ref{tab:sys} measures both families against
eager PyTorch on one GH200. At the configuration of our experiments
($L{=}8$, $G{=}32$), training steps are $1.35\times$ faster in fp32 and
$1.5$--$1.7\times$ faster under bf16 autocast at $14$--$25\%$ less peak
memory, inference is $1.5$--$1.7\times$ faster at $20$--$26\%$ less, and
the largest sample one GPU processes grows from $8.4$M to $10.5$M points
in fp32 and from $10.5$M to $14.7$M in bf16. The saving is structural: our memory is flat in $G$ ($15.42$\,GB
at $G{=}16$, $15.43$\,GB at $G{=}128$) where eager's rises by $0.17$\,GB
per slice (Figure~\ref{fig:systems}b), and at a fixed $95$\,GB budget we
train $48$ layers where eager fits $32$ (bf16) and $32$ where it fits
$16$ (fp32). At large slice counts the $G$-blocked kernels turn this into
a qualitative change: one layer holds $2$\,GB whether $G$ is $48$ or
$1024$, eager grows linearly and runs out of memory at $G{=}512$,
$N{=}1$M, and with bf16 tensor-core dots the fused layer is
$4.0$--$6.5\times$ faster. Time at large $G$ is a dot-throughput
question: a point does eight times the multiply-adds at $G{=}256$ that it
does at $G{=}32$, and Triton's exact fp32 dot is an FMA path at roughly
$15$\,TFLOP/s, so in the default mode the blocked kernels are
compute-bound and land near eager ($0.85$--$1.36\times$), whereas
tensor-core dots pull them ahead by $3$--$6\times$ in bf16 and up to
$1.9\times$ for fp32 inputs via a three-term tf32 split. Against \emph{compiled}
eager the kernels keep $1.24$--$1.26\times$, and they compose with
activation checkpointing and point-sharding (Appendix~\ref{app:systems}).

\begin{table}[t]
\caption{\textbf{\flashslice{} against eager PyTorch on one GH200}
(PyTorch~2.5, Triton~3.0; medians of $\ge10$ iterations measured in one
job; no compilation or checkpointing; ``OOM'' $=$ does not fit in
$95$\,GB). Top: the model of our experiments ($L{=}8$, $C{=}256$, $H{=}8$,
$G{=}32$), single-tile kernels, exact fp32 dots in fp32 and bf16
tensor-core dots in bf16. Bottom: one layer ($H{=}8$, $D{=}32$) at large
slice counts, $G$-blocked kernels, bf16 inputs and dots.}
\label{tab:sys}
\begin{center}\footnotesize
\setlength{\tabcolsep}{4.5pt}
\begin{tabular}{llrrrcrrc}
\toprule
 & & & \multicolumn{3}{c}{time (ms)} & \multicolumn{3}{c}{peak memory (GB)} \\
\cmidrule(lr){4-6}\cmidrule(lr){7-9}
 & & $N$ & eager & ours & speedup & eager & ours & saved \\
\midrule
\multicolumn{9}{@{}l}{\emph{full model, $G{=}32$, single-tile kernels}} \\
training  & fp32 & $262$k  & $171$  & $126$  & $1.35\times$ & $28.0$ & $24.0$ & $14\%$ \\
training  & bf16 & $262$k  & $143$  & $93.1$ & $1.53\times$ & $20.4$ & $15.4$ & $24\%$ \\
training  & bf16 & $1$M    & $585$  & $345$  & $1.70\times$ & $81.3$ & $61.3$ & $25\%$ \\
inference & fp32 & $8.4$M  & $1944$ & $1160$ & $1.68\times$ & $80.5$ & $64.5$ & $20\%$ \\

inference & bf16 & $12.6$M & OOM    & $1411$ & ---          & ---    & $66.7$ & ---    \\
\midrule
\multicolumn{9}{@{}l}{\emph{one layer, bf16, $G$-blocked kernels}} \\
training  & $G{=}256$  & $262$k & $47.5$  & $11.9$ & $3.98\times$ & $13.7$ & $2.0$ & $85\%$ \\
training  & $G{=}256$  & $1$M   & $293.7$ & $45.5$ & $6.45\times$ & $54.5$ & $8.1$ & $85\%$ \\
training  & $G{=}512$  & $262$k & $90.6$  & $20.0$ & $4.53\times$ & $26.7$ & $2.0$ & $92\%$ \\
training  & $G{=}512$  & $1$M   & OOM     & $80.3$ & ---          & ---    & $8.1$ & ---    \\
training  & $G{=}1024$ & $262$k & $180.2$ & $37.6$ & $4.79\times$ & $52.7$ & $2.1$ & $96\%$ \\
inference & $G{=}256$  & $4.2$M & $157.3$ & $53.2$ & $2.95\times$ & $68.0$ & $8.3$ & $88\%$ \\
\bottomrule
\end{tabular}
\end{center}
\end{table}
\section{Discussion}
\paragraph{Summary.} Our aim was to present a thorough empirical and theoretical analysis of the widely used Transolver architecture in order to uncover the mechanisms that underwrite its accuracy in operator learning on arbitrary domains. By very careful and controlled ablations on a challenging suite of nine benchmarks in fluid dynamics and industrial aerodynamics, we demonstrated that i) token attention can be completely removed and replaced by a constant linear map implying that the Transolver does not require a Transformer ii) the slicing/deslicing global mixing operators are the key source of accuracy and removing them leads to collapse in performance as pointwise MLPs cannot recover global operators and iii) the coupling has to be followed by pointwise processing, and how often it has to be repeated grows with the nonlocality of the operator. We use the theory of averaging neural operators to account for these findings in terms of model expressivity and test its qualitative predictions, providing a clear understanding of the causal mechanisms for Transolver's function.  Finally, we provide a novel and significantly more efficient implementation of the slicing operator by proposing \flashslice{} as a computational kernel. In practice, we recommend to drop the token attention, to keep slice/deslice with the point stream and a pointwise decoder, to set the number of coupling rounds by the nonlocality of the task, and to run the coupling with \flashslice{}. 
\paragraph{Related Work.} Transolver was originally proposed in \cite{wu2024transolver} and several extensions have been provided by \citep{luo2025transolverpp,zhou2026transolver3,adams2025geotransolver}, among others. The focus of most of these papers is on modifications of the token attention module of Transolver in order to improve performance. On the other hand, there is very little literature on the analysis and understanding of the model itself with the notable exception of \citet{hu2026linearno}, who observe
that the physics-attention is a special case of linear attention and find that
deleting the slice attention often improves accuracy on six standard PDE
benchmarks. However, they attribute this fact to the tying of the projections, and propose LinearNO,
which unties the projections and removes the attention; \citet{lrsa2026}
reach a compatible reading from a low-rank-kernel viewpoint. As argued in Sec.~\ref{sec:experiments}, we also investigate this claim and this is clearly not the case. Rather, tying is orthogonal to the redundancy of the attention: with or without it, removing the attention costs nothing measurable, and all four cells of Figure~\ref{fig:controls}(a) lie within $1.05\%$ of one another, as do those of a second benchmark (Table~\ref{tab:rb-factorial}); their layer is moreover $2.4$--$16\times$ more expensive to train (Table~\ref{tab:rb-cost}). Encode-once models such as LNO \citep{wang2024lno} and LaMO \citep{tiwari2025lamo} fit our analysis: a pointwise decoder closes most of the \sliceonce{} gap, and a state-space mixer does not improve on attention. Moreover, to the best of our knowledge, our paper is the first to provide theoretical justifications of the extensive empirical findings, leading to a holistic understanding of the mechanisms that govern Transolver's behavior.

\paragraph{Limitations and Future Work.} Our paper focuses on one, albeit very widely used, neural operator architecture and the findings about the role of global mixing vs. attention might not carry over to other unrelated neural operators. Moreover, our theory is based on model expressivity. It does not take generalization and optimization into account, nor are the bounds sharp. Yet, we could use the theory to suggest hypotheses which were validated with extensive empirical evidence. Going forward, we would like to leverage the insights presented here to develop novel neural operators that have optimized and non-redundant kernels, leading to further accuracy on AI for industrial engineering. 

% Acknowledgements: uncomment and fill in before the upload.
 \section*{Acknowledgements}
The contribution of S. Mishra to this work was supported in part by the DOE SEA-CROGS project (DE-SC-0023191).
% ...

\newpage
\bibliography{iclr2027_conference}
\bibliographystyle{iclr2027_conference}

\clearpage
\appendix
% appendix floats are numbered per section: Figure A.1, Table B.1, ...
\counterwithin{figure}{section}
\counterwithin{table}{section}
\addtocontents{toc}{\protect\setcounter{tocdepth}{2}}
\renewcommand{\contentsname}{Appendix}
\tableofcontents
\clearpage
% appendix layout: floats stay inside their section, and float pages are
% filled from the top instead of being centred vertically
\let\appsection\section
\renewcommand{\section}{\FloatBarrier\appsection}
\makeatletter
\setlength{\@fptop}{0pt}
\setlength{\@fpbot}{0pt plus 1fil}
\makeatother

\section{Slice/deslice as a learnable basis: statements and proofs}
\label{app:theory}

This section of the appendix provides details for \S\ref{sec:theory}: the
placement of Eq.~(\ref{eq:lowrank}) within the nonlocal neural operator
template, the proofs of Proposition~\ref{prop:p1},
Corollary~\ref{cor:mlponly-depth} and Proposition~\ref{prop:p3} and further implications and limitations of these results. 

\subsection{Equation~(\ref{eq:lowrank}) as a low-rank nonlocal neural operator}
\label{app:theory-nno}

The $\ell$-th hidden layer of the nonlocal neural operator (NNO) of
\citet{kovachki2023neural} takes the form
$\mathcal{L}_\ell(v)(x)=\sigma\big(W_\ell v(x)+b_\ell+
\sum_{m}\langle T_{\ell,m}v,\psi_{\ell,m}\rangle_{L^{2}}\,\phi_{\ell,m}(x)\big)$;
Eq.~(\ref{eq:lowrank}) is the instance with $M=G$,
$\psi_{\ell,g}=\tilde s_g:=s_g/\!\int_\Omega s_g\,\mathrm{d}\mu$, $\phi_{\ell,g}=s_g$ and $T_{\ell,gh}=A_{gh}$,
the special case that \citet[\S4.2]{kovachki2023neural} call the
\emph{low-rank neural operator}. Two features distinguish
physics-attention within this template, and both matter below: (i)~the
``basis'' $\{s_g\}$ is a learned, data-dependent partition of unity
re-estimated from $f^{(\ell-1)}$ at every layer, not a fixed dictionary, nor is it a basis in a rigorous sense, we simply use the term here for narrative convenience;
(ii)~the coupling is interleaved with a pointwise nonlinearity and a
persistent full-resolution pointwise stream, rather than just applied once.

The averaging neural operator (ANO) of \citet{lanthaler2023nonlocality}
is the simplest member of the family: a single hidden layer with $G=1$ and
a constant, non-adaptive slice function, so that the normalized token is
the global mean and Eq.~(\ref{eq:ano}) results, with pointwise lifting
and projection networks $R,Q$ of width $J$ (their $d_c$). 

As \cite{lanthaler2023nonlocality} argue, this minimal averaging neural operator suffices for universality. Their precise statement is reproduced verbatim:
\begin{theorem}[Universality of the ANO; \citealt{lanthaler2023nonlocality}, Thms.~2.1--2.2]
\label{thm:ano}
Let $\Omega\subset\mathbb{R}^d$ be a bounded Lipschitz domain and
$\Psi^\dagger:\mathcal{X}\to\mathcal{Y}$ a continuous operator between
$C^{s}$ or $W^{s,p}$ spaces of functions on $\Omega$. For every compact
$\mathsf{K}\subset\mathcal{X}$ and $\epsilon>0$ there exist $J\in\mathbb{N}$
and networks $R,Q$ of width $J$ such that $\Psi$ of the
form~(\ref{eq:ano}) satisfies
$\sup_{u\in\mathsf{K}}\|\Psi^\dagger(u)-\Psi(u)\|_{\mathcal{Y}}\le\epsilon$.
\end{theorem}

The proof of
Theorem~\ref{thm:ano} (\citealt{lanthaler2023nonlocality}, \S2.5)
proceeds in two steps. (1)~Any target operator is $\epsilon/2$-approximated
by a finite sum $\sum_{j=1}^{J}\alpha_j(u)\,\eta_j$ of an \emph{encoder},
$J$ continuous nonlinear functionals
$\alpha_j:L^{1}(\Omega;\mathbb{R}^{k})\to\mathbb{R}$, against a fixed
\emph{decoder} basis $\eta_j\in\mathcal{Y}$. (2)~Each functional is
$\epsilon/2J$-approximated by one global average of a pointwise-lifted
input, $\alpha_j(u)\approx\sigma\big(\int_\Omega R_j(u(y),y)\,\mathrm{d}y\big)$,
because averaging against $R_j(u(y),y)=u(y)\cdot\xi_{j}(y)$ for a smooth
$\xi_j$ recovers a mollified inner product with the $j$-th element of a
basis of $L^{2}(\Omega)$. Step~(2) is where all of the nonlocality is
spent; step~(1), the decoding, is pointwise. The encoding width
$J=J(\epsilon)$ is generically unbounded as $\epsilon\to0$: nothing in
the theorem caps how much must be encoded through the single average, and
a small fixed $G$ is not the regime it addresses directly.
\S\ref{app:theory-p3} locates where that width lives in a Transolver
stack, namely in the persistent point stream rather than in the slice
count.

\subsection{Proof of Proposition~\ref{prop:p1}, and what it implies}
\label{app:theory-p1}

\begin{proof}[Proof of Proposition~\ref{prop:p1}]
Take the $G$ slice logits identically zero, so that
$s_1=\cdots=s_G\equiv1/G$, a legal partition of unity reached exactly
rather than in a limit. Every normalized token is then the global mean,
$z_h=\big(\tfrac1G\int_\Omega v\,\mathrm{d}\mu\big)\big/\big(\tfrac1G|\Omega|\big)=\bar v$,
and with the constant core scaled so that $\sum_{g,h}A_{gh}/G=1$,
Eq.~(\ref{eq:lowrank}) reduces exactly to the ANO coupling
$(\mathcal{K}v)(x)=\bar v$. The pointwise MLP sublayers that precede and
follow it realize $R$ and $Q$ up to the usual universal approximation of
ordinary networks (\citealt{lanthaler2023nonlocality}, Lemma~A.5, after
\citealt{pinkus1999approximation}), the coordinate $x$ being available to
$Q$ through the residual stream, so Theorem~\ref{thm:ano} applies
verbatim. The construction uses no property of $G$ and is available at
every $G\ge1$.
\end{proof}

Proposition~\ref{prop:p1} is deliberately not the range argument that a
similar claim tempts one to make. It is true, and used in
\S\ref{sec:anatomy}, that the range of $\mathcal{K}$ is
$\mathrm{span}\{s_g\}$ regardless of $A$, so that a content-dependent
core cannot enlarge what a single coupling expresses. That alone shows
only that the attention cannot \emph{enlarge} the layer's range, not that
it is dispensable, since a constant and a content-dependent core
parameterize different maps within the same range; the density statement
is what removes the expressivity-based reason to expect the attention to
be necessary.

What the proposition does not imply is the stronger reading that the
attention therefore costs nothing at the budget a practitioner trains at.
Universality is shared by the two classes: \notoken{}'s class is contained in
\baselinev{}'s, both are universal, and this density cannot separate them; and the
construction spends unbounded pointwise width, which no experiment has.
Our own results supply the cautionary case: by the same argument the
frozen-basis class is universal (\S\ref{app:theory-p3}), and \frozensl{}
nonetheless costs up to $17\%$ (Table~\ref{tab:main}). Universality predicts
the absence of a ceiling, not the absence of a cost. That the cost of
removing the attention is indistinguishable from zero on nine benchmarks
at matched capacity is an empirical finding; the theory makes it
unsurprising rather than deriving it from first principles.

\subsection{Locality, and the proof of Corollary~\ref{cor:mlponly-depth}}
\label{app:theory-p2}

\begin{lemma}[Locality is closed under composition]
\label{lem:locality}
Let $\Phi_1,\ldots,\Phi_L$ be pointwise operators,
$\Phi_\ell(v)(x)=\sigma(W_\ell v(x)+b_\ell)$, possibly with pointwise
residual connections. Then $\Phi_L\circ\cdots\circ\Phi_1$ is pointwise:
its output at $x$ depends on the input only through the input's value at
$x$.
\end{lemma}
\begin{proof}
By induction on $L$: if $\Phi_1(v)(x)$ depends on $v$ only through $v(x)$
and $\Phi_2$ is pointwise, then $\Phi_2(\Phi_1(v))(x)$ depends on
$\Phi_1(v)$ only through $\Phi_1(v)(x)$, hence on $v$ only through $v(x)$.
\end{proof}

\begin{proof}[Proof of Corollary~\ref{cor:mlponly-depth}]
\mlponly{} is the pointwise MLP-with-residual stack that remains once
$\mathcal{K}$ is removed, so by Lemma~\ref{lem:locality} it is pointwise
at every depth $L$ and every width. The infimum in
Definition~\ref{def:localitygap} is taken over a class that is closed in
the topology of uniform convergence on $\mathsf{K}$: if $\Phi_n$ are
pointwise and $\Phi_n\to\Psi$ uniformly on $\mathsf{K}$, then any
$u_1,u_2\in\mathsf{K}$ with $u_1(x)=u_2(x)$ satisfy
$\Phi_n(u_1)(x)=\Phi_n(u_2)(x)$ for every $n$, hence
$\Psi(u_1)(x)=\Psi(u_2)(x)$, and $\Psi$ is pointwise. The gap is therefore
a genuine obstruction: every pointwise operator, \mlponly{} at any size
included, has error at least $\delta_{\mathrm{loc}}(\Psi^\dagger)$ on
$\mathsf{K}$.
\end{proof}

Following \citet[\S2.1]{lanthaler2023nonlocality}, $\delta_{\mathrm{loc}}>0$
for every operator with genuine nonlocal dependence, the shift
$\Psi^\dagger(u)(x)=u(x+h)$ being the canonical example. More relevant
here, no pointwise map can represent the solution operator of an elliptic
boundary value problem, whose Green's function has support equal to all
of $\Omega$. This is exactly the pressure field in incompressible or
low-Mach aerodynamics. The \mlponly{} penalty of Table~\ref{tab:main} is
thus a finite-sample, finite-capacity estimate of $\delta_{\mathrm{loc}}$,
and the per-task and per-channel ranking of that penalty in
\S\ref{sec:experiments} is a ranking by $\delta_{\mathrm{loc}}$: largest
on DrivAerML, smallest for Taylor--Green, whose
short-horizon map is closest to the identity restricted to a neighborhood
of $x$, and within every aerodynamic task larger for the pressure than
for the velocity or shear. One caveat keeps this an ordering rather than a
precise prediction: velocity is not local either, being coupled to the same global
pressure through $\nabla p$. The empirical results agree with this interpretation
where the comparison is cleanest, on SHIFT-SUV (same cars, fixed condition,
geometry the only input): the pressure/velocity separation in the volume
($4.5\times$ against $2.6\times$) is narrower than the pressure/shear
separation on the surface ($4.2\times$ against $1.3\times$), where the
companion quantity is a wall trace far more nearly determined by local
geometry (Table~\ref{tab:pervar-l1}). On DrivAerNet++ and DrivAerML the two
separations are about equal, and the transonic SHIFT-Wing volume is
the exception, at $14.8\times$ against $1.6\times$ for the velocity. Quantifying $\delta_{\mathrm{loc}}$ directly,
for instance through the decay of the governing Green's function, rather
than reading it off a trained \mlponly{}, is a natural next step we leave
open.

Corollary~\ref{cor:mlponly-depth} also yields a prediction beyond what we
test: parameter-matched \mlponly{} models at $L=4,8,16,32$ should show no
improvement trend at all, unlike \notoken{} or \baselinev{}. We suggest
this depth sweep as a companion to the \mlpwide{} capacity control of
Figure~\ref{fig:controls}(c), isolating locality from capacity a second,
orthogonal way.

\subsection{Proof of Proposition~\ref{prop:p3}, the looseness of the budget, and the frozen basis}
\label{app:theory-p3}

Theorem~\ref{thm:ano} reaches universality with a single coupling but a
possibly unbounded encoding width $J$. It is tempting to say that Transolver,
which holds $G$ fixed and grows $L$ instead, trades depth for that width.
The results do not support this trade: what a coupling
transmits is bounded not by $G$ but by $G$ times the channel width, and
the encoding is carried, between couplings, in the persistent point
stream.

\begin{proof}[Proof of Proposition~\ref{prop:p3}]
(i)~is the construction of Proposition~\ref{prop:p1} read at finite
width: the uniform basis $s_g\equiv1/G$ makes the coupling the global mean
in all $c$ channels at once, the preceding pointwise MLP realizes the ANO
lifting $R$ into $J\le c-k$ of them, and the following one realizes $Q$
from those channels together with the coordinates carried in the reserved
$k$. No step varies with $G$. (ii)~The tokens are
$z_h=\int s_h v\,\mathrm{d}\mu/\int s_h\,\mathrm{d}\mu\in\mathbb{R}^{c}$
for $h=1,\ldots,G$, that is $G\,c$ scalars; deslice returns
$\sum_{g,h}s_g(x)A_{gh}z_h$, a linear image of them at each $x$, and by
Lemma~\ref{lem:locality} no pointwise map applied afterwards adds
information about $v$ at a distance. In \sliceonce{} those $G\,c$ scalars
are produced once and every later layer acts on the $G$ tokens alone, so
the bound is independent of that depth; the point-level skip that
survives in \sliceonce{} carries the pointwise part of the input, not
further nonlocal content.
\end{proof}

\paragraph{The budget is an upper bound, and a loose one.} At our
configuration $G\,c=32\times256\approx8\times10^{3}$ scalars pass through
a single coupling, which is not obviously too few, so part~(ii) does not
on its own account for the size of the \sliceonce{} penalty
($2.2$--$6.8\times$). A second structural difference does, and
Theorem~\ref{thm:ano} makes it explicit: the ANO needs a universal
pointwise decoder $Q$ acting at full resolution on (encoding, $x$).
\sliceonce{} spends its depth transforming the $G$ tokens and then reads
out affinely at the points, so what it realizes is a nonlinear map of the
encoding followed by a map that is linear in the tokens with
$x$-dependent coefficients $s_g(x)$; the second ingredient of the theorem
is missing, not merely narrow. What part~(ii) establishes is the
qualitative asymmetry, budgets $0$, $G\,c$ and $L\,G\,c$, and the
prediction that only the third improves with depth; the magnitude is not
something we derive. The decoder control of Table~\ref{tab:rb-coupling}
tests this reading: supplying four pointwise MLP blocks after the single
deslice reduces the \sliceonce{} penalty from $2.87\times$ to $1.28\times$
(SHIFT-SUV surface) and from $6.31\times$ to $3.02\times$ (DrivAerML surface).
What remains on DrivAerML is removed by coupling more often
($1.38\times$, $1.12\times$ and $1.02\times$ at $K=1$, $2$ and $4$).

\paragraph{The frozen basis.} Part~(i) uses a constant, non-adaptive
basis, $s_g\equiv1/G$ fixed once as in \frozensl{}, and still attains
universality, because all the adaptivity it needs comes from the
pointwise networks between rounds. This matches the ANO itself, whose
single nonlocal ingredient is a fixed uniform average. Read literally, it
predicts that a frozen basis is not subject to a hard ceiling analogous
to \mlponly{} or \sliceonce{}: given enough depth, a fixed basis reused
at every layer can still reach any target. What re-estimating $\{s_g\}$
from $f^{(\ell-1)}$ buys is efficiency: a data-dependent basis can spend
each of its $L$ rounds probing whichever $G$-dimensional projection of
the current field is most informative, whereas a frozen basis probes the
same projections every round and relies on the pointwise nonlinearity
between rounds to make each repetition informative. This is the
distinction of \S\ref{app:theory-p1}, no ceiling but a cost, and the cost
is what Table~\ref{tab:main} measures, at a few percent to $17\%$. It also
names the discriminating experiment: under a depth sweep the \mlponly{}
and \sliceonce{} gaps should persist at every $L$, while the \frozensl{}
gap should close. For \mlponly{}, the depth sweep is reported in
Table~\ref{tab:rb-depth}: its gap persists from $L=4$ to $L=16$. The depth
sweeps of \sliceonce{} and \frozensl{} are not among the runs we report.

\subsection{From the continuum to the mesh}
\label{app:theory-disc}

Theorem~\ref{thm:ano} and Propositions~\ref{prop:p1}--\ref{prop:p3} are
stated for the continuum operator $\mathcal{K}$; the implementation sees
only $N$ mesh points. Writing the discrete slice weights as quadrature
weights, $w_{i,g}\,\omega_i\approx s_g(x_i)\,\mathrm{d}\mu(x_i)$ for a per-point
volume element $\omega_i$ with $\sum_i \omega_i=|\Omega|$, the discrete coupling
$\sum_i w_{i,g}v_i/\sum_i w_{i,g}$ is a Riemann-sum quadrature of
$\int_\Omega s_g(y)v(y)\,\mathrm{d}\mu(y)/\int_\Omega s_g(y)\,\mathrm{d}\mu(y)$,
with error controlled, for Lipschitz $s_g$ and $v$, by the fill distance
of $\{x_i\}$ in $\Omega$, the same collocation argument used to pass from
continuum to discrete FNOs in \citet{kovachki2021universal}. We do not
develop explicit rates here, since our point clouds are highly
non-uniform unstructured surface and volume meshes for which
fill-distance bounds would be conservative. Instead we note that the
subsampling-noise measurement of Table~\ref{tab:noise} is an empirical
measurement of exactly this discretization error, and that its magnitude
($\le0.18\%$ relative, three orders of magnitude below every ablation
effect in Table~\ref{tab:main}) is consistent with the discretization
term being negligible relative to the architectural effects this paper
studies.

\subsection{The learned basis, empirically}
\label{app:basis}
We compare the slice functions learned by independently trained
\baselinev{} and \notoken{} on DrivAerML. At the last layer, both models produce spatially coherent patterns corresponding to recognizable parts of the car, including the greenhouse and upper body, front fascia, and flanks (Figure~\ref{fig:basis}). Matching slices by cosine similarity over points gives similarities of $0.82$, $0.72$, and $0.69$ for the three closest pairs.

At the first layer, assignments are nearly uniform in both models:
the mean per-point entropies are $3.47$ and $3.46$ nats, close to the maximum $\ln 32 = 3.47$, and all $32$ slices are effectively used. The spatial patterns therefore develop in later layers.
At the last layer, \notoken{} uses more effective slices than
\baselinev{} ($20.4$ vs.\ $15.4$), but has lower mean per-point entropy ($2.22$ vs.\ $2.48$ nats). Thus, it uses more slices overall while making more concentrated assignments at individual points. This is consistent with the interpretation in \S\ref{app:theory-p3}: when the core is held constant, the learned
basis may take on a greater role in input-dependent routing.

\begin{figure}[htbp]
\centering
\includegraphics[width=\linewidth]{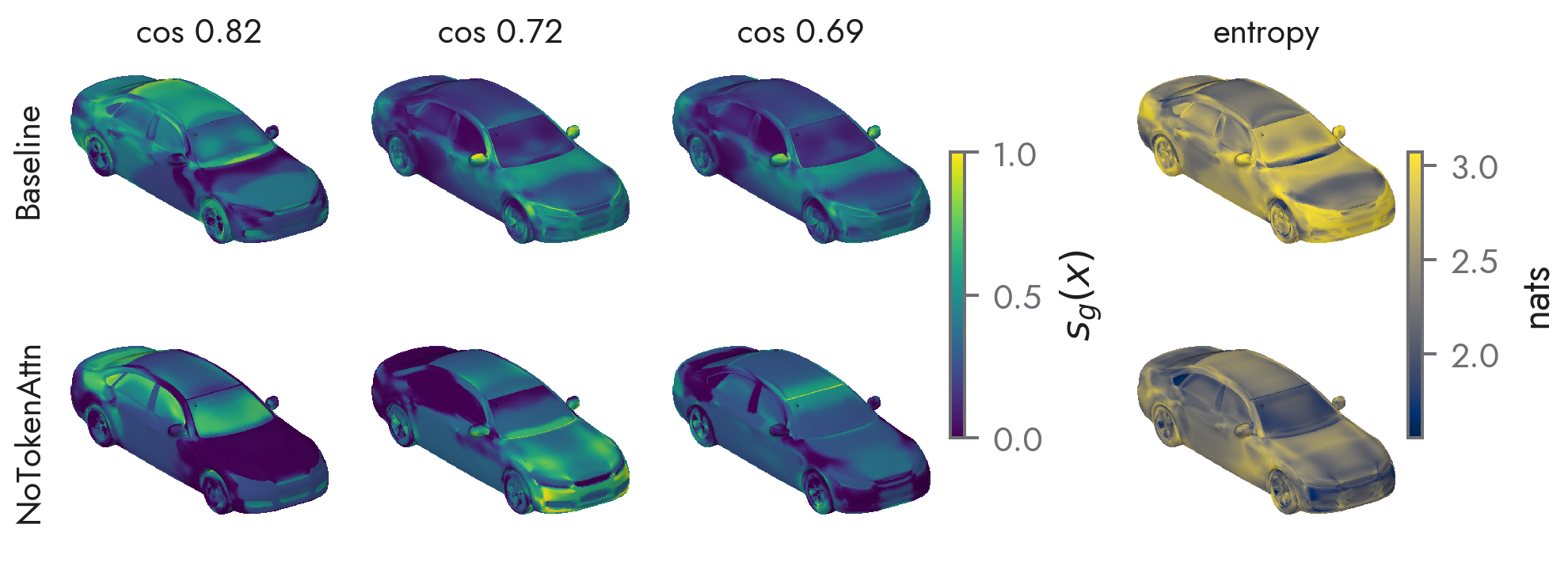}
\caption{\textbf{Learned slice functions with and without attention.}
Last-layer slice functions $s_g(x)$ for \baselinev{} (top) and
\notoken{} (bottom), shown on the DrivAerML test car from
Figure~\ref{fig:fields}. The first three columns show the closest
pairs under Hungarian matching by cosine similarity over points
($0.82$, $0.72$, and $0.69$). Each slice is normalized by its own
maximum. The rightmost column shows per-point assignment entropy
on a shared color scale.}
\label{fig:basis}
\end{figure}

\subsection{Summary}
\label{app:theory-summary}

Table~\ref{tab:budgets} in Section~\ref{sec:theory} separates what the theory settles from what only
experiments can.

\section{Datasets, metrics, and hyperparameters}
\label{app:datasets}

\paragraph{Benchmarks.}
Table~\ref{tab:datasets} summarizes the nine benchmark configurations, including data splits, training budgets, point counts after the \texttt{physical\_domain} crop, and input and output fields. Splits follow the standard protocol for each dataset. During training, a fresh random subset of points is drawn each time a sample is loaded, using the ratios listed in the table.
The measured relative variation in the evaluation metric due to
subsampling is at most $0.18\%$ (Table~\ref{tab:noise}). All listed input features, including freestream conditions where
provided, are available at every point. SHIFT-SUV provides no freestream attributes and uses a single fixed flow condition, so its input features consist only of normals (surface) or signed distance (volume).

\begin{table}[htbp]
\caption{\textbf{Benchmark configurations.}
Training steps, split sizes, points per sample after cropping and
before subsampling, training subsampling ratios, and input/output
channels. $^{\ast}$Evaluated at the common checkpoint of 140k steps, reached within one 12-hour run using a 250k-step cosine schedule (Appendix~\ref{app:budgets}).}
\label{tab:datasets}
\begin{center}
\scriptsize
\setlength{\tabcolsep}{2.3pt}
\resizebox{\linewidth}{!}{%
\begin{tabular}{lcccccll}
\toprule
Benchmark & steps & train/val/test & points/sample & subsample & in$\to$out ch. &
inputs & outputs \\
\midrule
Taylor--Green & 250k & 1.55M/19.3k/38.6k & $32^{3}$ & $1.0$ & $5{\to}5$ &
  state at $t$ (+ lead time) & state at $t{+}\Delta t$ \\
SHIFT-Wing surface & 250k & 8377/1000/1000 & 1.6M & $0.1$ & $8{\to}4$ &
  normals, freestream (varying) & $p$, $\tau$ \\
SHIFT-Wing volume & 140k$^{\ast}$ & 8377/1000/1000 & 5.85M & $0.046$ & $6{\to}4$ &
  SDF, freestream (varying) & $p$, $u,v,w$ \\
DrivAerNet++ surface & 184k & 5825/1148/1154 & 459k & $0.5$ & $8{\to}4$ &
  normals, freestream (const.) & $p$, $\tau$ \\
DrivAerNet++ volume  & 180k & 5825/1148/1154 & 11.0M & $0.02$ & $6{\to}4$ &
  SDF, freestream (const.) & $p$, $u,v,w$ \\
SHIFT-SUV surface & 50k & 798/100/100 & 2.57M & $0.092$ & $3{\to}4$ &
  normals \emph{only} & $p$, $\tau$ \\
SHIFT-SUV volume  & 50k & 798/100/100 & 46.5M & $0.0059$ & $1{\to}4$ &
  SDF \emph{only} & $p$, $u,v,w$ \\
DrivAerML surface & 50k & 400/34/50 & 7.8M & $0.03$ & $8{\to}4$ &
  normals, freestream (const.) & $p$, $\tau$ \\
DrivAerML volume  & 19k & 400/34/50 & 133M & $0.002$ & $6{\to}4$ &
  SDF, freestream (const.) & $p$, $u,v,w$ \\
\bottomrule
\end{tabular}}
\end{center}
\end{table}

\paragraph{Metric.} We report relative $L^1$ error, computed separately for each sample $s$ and output group $g$, then averaged: $\mathrm{err} = \mathrm{mean}_{s,g}\, \|\hat{u}_{s,g} -
u_{s,g}\|_{1} / \|u_{s,g}\|_{1}$. For the aerodynamic benchmarks, the groups are pressure and the velocity or wall-shear vector. For Taylor--Green, they are density and the velocity vector.

\paragraph{Model.}
The baseline uses a Transolver backbone with $8$ layers, hidden
width $256$, $8$ heads, $32$ slices, an MLP ratio of $2$, GELU
activations, and no dropout (approximately $3.87$M parameters).
\mlponly{} removes the attention sublayer ($2.25$M parameters),
and \mlpwide{} increases the MLP ratio to $4$ ($4.35$M).
\sliceonce{} applies slicing once, followed by an $8$-layer token
transformer and a single deslicing operation ($4.37$M).

\paragraph{The \untiedln{} variant.}
For \untiedln{}, we use the released layer implementation of
\citet{hu2026linearno}, integrated without changes into our training
framework (Appendix~\ref{app:linearno}).
Relative to the baseline, this layer gives deslicing its own
projection and temperature, removes token self-attention, and
normalizes the analysis weights with a separate softmax over
the $N$ points (Figure~\ref{fig:linearno}(a)).
The last change distinguishes \untiedln{} from combining
\untied{} and \notoken{}, which normalizes the per-point slice
weights by their sums over points.
The resulting model has $3.85$M parameters.

\paragraph{Training.}
We use AdamW with a learning rate of $10^{-3}$, cosine decay,
$1\%$ warmup, weight decay $5\times10^{-5}$, and gradient clipping
at $5.0$. Training uses a per-device batch size of $1$--$5$, depending on the benchmark. Each run uses data parallelism on a single node with $4$ H100 GPUs. Table~\ref{tab:datasets} lists the training step budgets.

\section{Experimental protocol}
\label{app:protocol}

\subsection{Training budgets}
\label{app:budgets}
Table~\ref{tab:datasets} lists the training budget for each benchmark. Variants use the same number of steps unless noted below.

We assess budget sensitivity on SHIFT-SUV by comparing results
from $10$k to $50$k steps. Over this range, the error ratio of
\mlponly{} to the baseline increases from $1.52$ to $1.62$ on
the surface and from $2.43$ to $3.24$ in the volume.
The corresponding ratios for \notoken{} remain close to one:
$0.996$--$1.009$ on the surface and $1.015$--$1.035$ in the volume.
Thus, on this benchmark, extending training does not close the
\mlponly{} gap or substantially change the \notoken{} comparison.

The following runs differ from the budgets in
Table~\ref{tab:datasets}:
\begin{itemize}
\item On DrivAerNet++, \mlponly{} reaches $250$k steps on the
surface and $204$k in the volume, compared with approximately
$184$k and $180$k for the baseline, under equal wall-clock budgets.
\item On DrivAerML volume, the \mlponly{} run ends at $16.5$k
steps after its validation error plateaus. At the common
checkpoint of $12.5$k steps, its error is $4.1$ times the baseline.
\item On SHIFT-Wing surface, \untied{} ends at $234$k of the
planned $250$k steps.
\end{itemize}

On SHIFT-Wing volume, all variants are evaluated at $140$k steps,
the last evaluation reached by every variant within one 12-hour
run. These runs use a $250$k-step cosine schedule.
On SHIFT-SUV volume, the error ratios of \mlponly{} and
\sliceonce{} to the baseline are still increasing at the
$50$k-step cutoff. Figure~\ref{fig:allcurves} shows the validation curves for all benchmarks.

\subsection{Implementation checks}
\label{app:verification}
We verify each ablation using model structure, checkpoint contents,
and parameter updates. Module and parameter-key checks confirm
that the intended components are present or absent.
For \notoken{} and \frozensl{}, we also check that parameters
excluded from the computation remain unchanged between training
checkpoints, while active parameters update.
Assertions during model construction check that the requested
ablation matches the instantiated architecture.
Baseline replicas provide a reference for run-to-run variation.

\subsection{Full training curves}
\label{app:curves}
\begin{figure}[htbp]
\centering
\includegraphics[width=\linewidth]{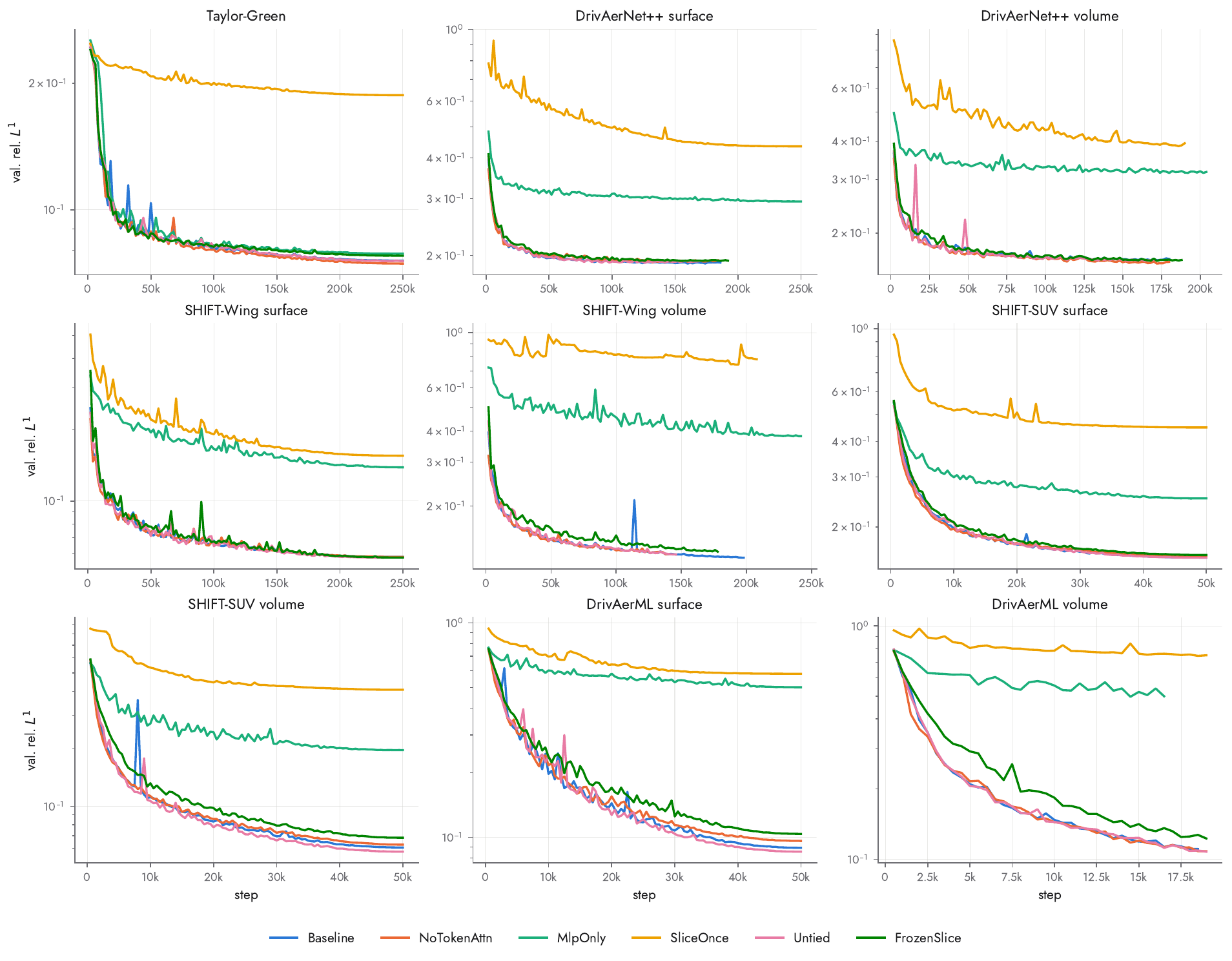}
\caption{Validation curves for all variants on the nine benchmark
configurations (seed 42). Training budgets and exceptions are
described in \S\ref{app:budgets}.}
\label{fig:allcurves}
\end{figure}

\section{Additional results}
\label{app:results}

\subsection{Pointwise predictions on aerodynamic benchmarks}
\label{app:mechanism}
A pointwise model receives the input features at each query point,
including freestream conditions where provided (Table~\ref{tab:datasets}), but cannot combine information across
points. Local geometric features, such as surface normals and
signed distance, do not generally determine the surrounding flow:
the field also depends on the rest of the body.
This dependence remains when the inflow is fixed but the geometry
varies between samples. Consistent with this, the largest
\mlponly{} error increase occurs on DrivAerML, which uses fixed
inflow, rather than on the varying-inflow SHIFT-Wing benchmark
(Table~\ref{tab:main}).

On the DrivAerML example in Figure~\ref{fig:fields}, \mlponly{}
smooths the wheel-arch suction and front stagnation patterns,
with errors around the wheels, mirror, and A-pillar.
The per-variable results in Figure~\ref{fig:controls}(d) show
a larger relative error increase for pressure than for velocity
or wall shear across the aerodynamic benchmarks.
These observations are consistent with a loss of information
about the overall geometry.
\sliceonce{} shows a different distribution of errors across
fields, suggesting that its limitations are not fully explained
by the pointwise model's lack of spatial coupling.

\begin{figure}[htbp]
\centering
\includegraphics[width=\linewidth]{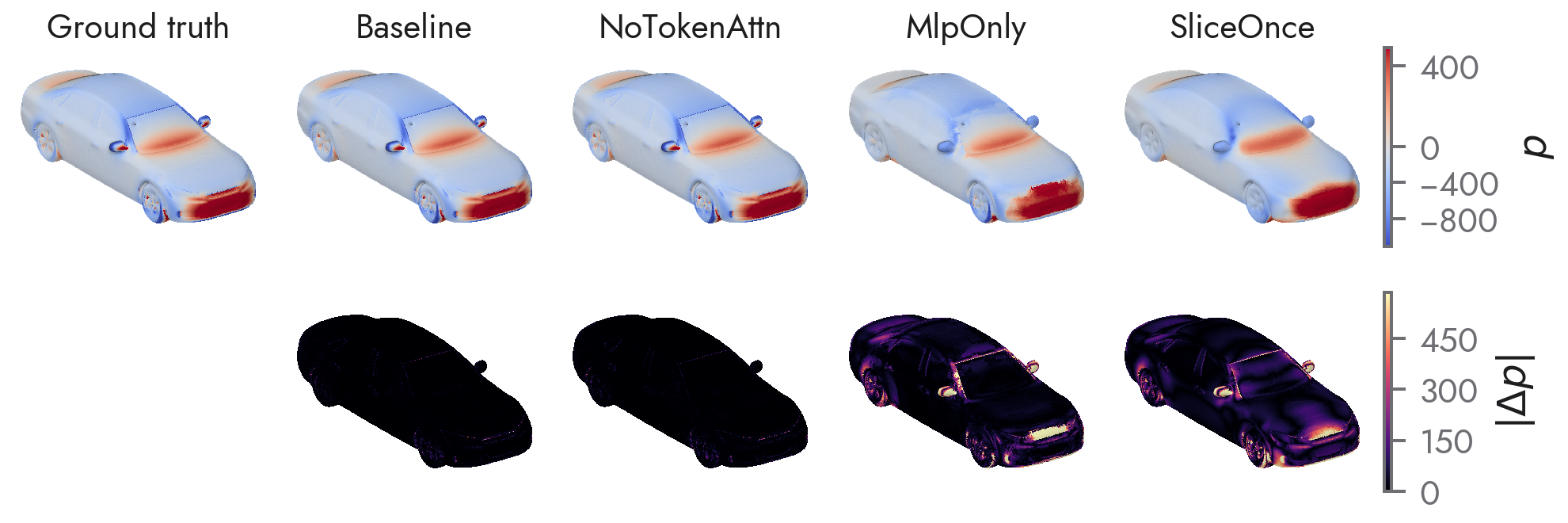}
\caption{\textbf{Surface pressure on a DrivAerML test car.} Top: ground
truth and the four variants. Bottom: each variant's
absolute error. The car is \texttt{run\_472}. Appendix~\ref{app:fieldset} states the protocol and repeats the figure on
the other eight benchmarks.}
\label{fig:fields}
\end{figure}

\subsection{Field predictions across benchmarks}
\label{app:fieldset}

Figures~\ref{fig:f-tg}--\ref{fig:f-dmlv-u} extend the comparison
in Figure~\ref{fig:fields} to the remaining benchmarks and output
fields. For each aerodynamic benchmark, we show pressure and
either wall shear or velocity; for Taylor--Green, we show velocity.
On the selected aerodynamic samples, the error of \mlponly{}
relative to the baseline is approximately $1.8$--$20$ times
larger for pressure and $1.3$--$5.1$ times larger for the
corresponding vector field.

\paragraph{Sample selection and evaluation.}
For each benchmark, we select the sample with median baseline
relative $L^1$ error among the first fifty test samples
(the full test split for DrivAerML).
All evaluations use fp32 and the training subsampling ratios
in Table~\ref{tab:datasets}.
For visualization, we pool predictions from several subsample
seeds to obtain a denser point cloud while retaining the
training point density in each model evaluation.

\paragraph{Rendering.}
Surfaces are rendered by depth-buffered point splatting with
depth-based shading. Volume sections interpolate points within
a thin slab around the displayed plane; the body is grey,
and regions without sufficient point coverage are white.
Within each figure, prediction panels share a color scale,
and error panels share a separate scale spanning all variants.
Reported errors are for the selected sample and therefore
differ from the test-set averages in
Table~\ref{tab:pervar-l1}.

\begin{figure}[tbp]
\centering
\includegraphics[width=\linewidth]{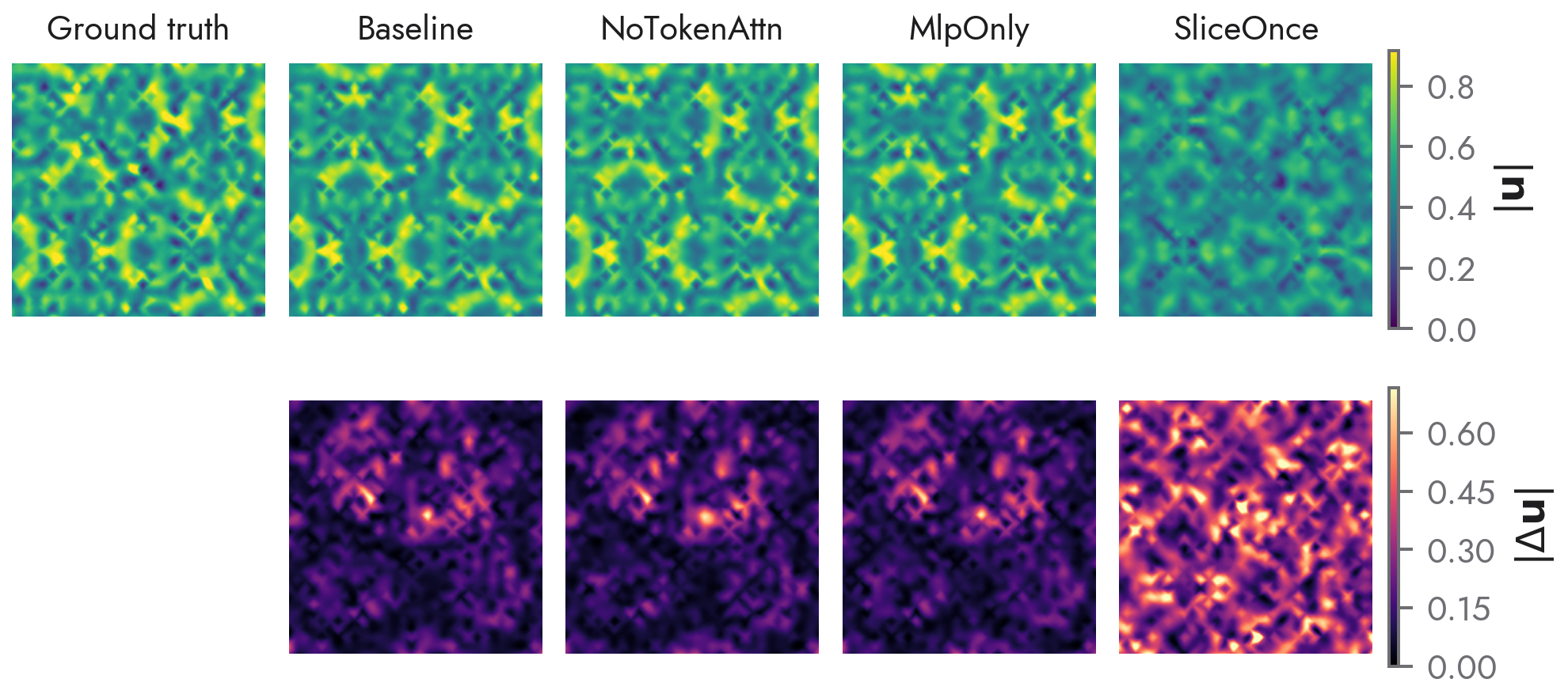}
\caption{\textbf{Taylor--Green velocity magnitude} on a mid-plane of the
$32^3$ cube, predicted from $t=1.28$ to $t=1.36$. Relative $L^1$ errors:
\baselinev{} $0.149$, \notoken{} $0.139$, \mlponly{} $0.148$,
\sliceonce{} $0.450$. \sliceonce{} smooths the fine vortex structure.}
\label{fig:f-tg}
\end{figure}

\begin{figure}[p]
\centering
\includegraphics[width=\linewidth]{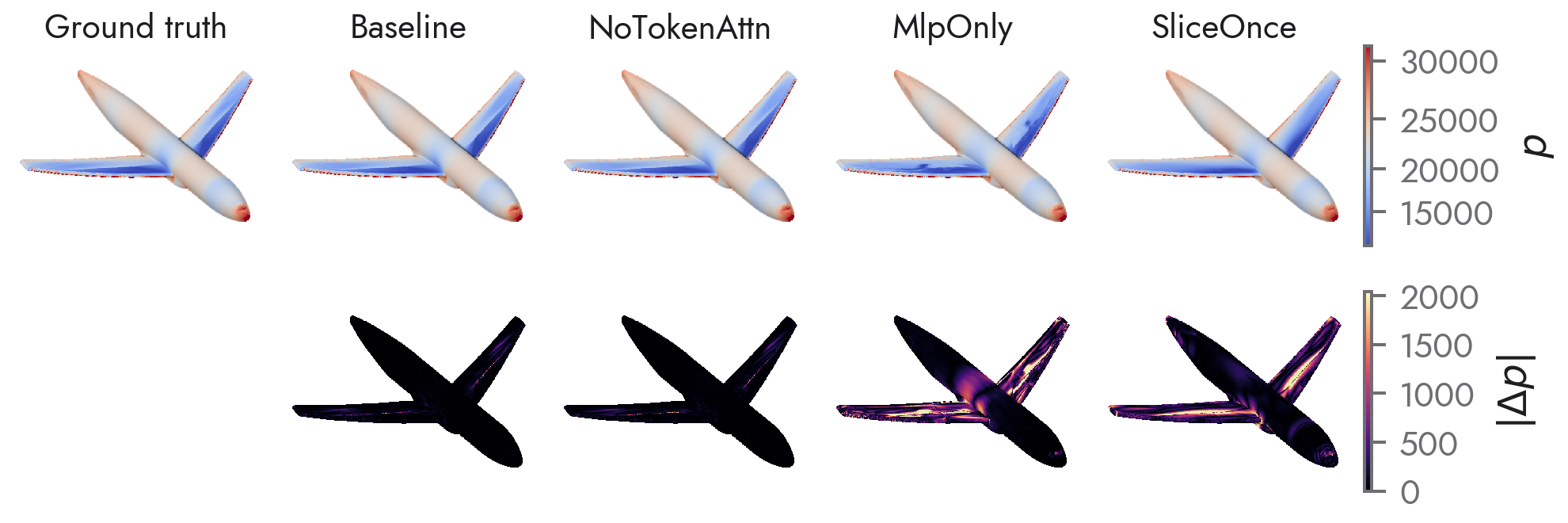}
\caption{\textbf{SHIFT-Wing surface pressure} for sample \texttt{016040}.
Relative $L^1$ errors: \baselinev{} $0.00085$, \notoken{} $0.00088$,
\mlponly{} $0.0098$, \sliceonce{} $0.0086$. The latter two smooth the
leading-edge pressure pattern.}
\label{fig:f-sws}
\end{figure}

\begin{figure}[p]
\centering
\includegraphics[width=\linewidth]{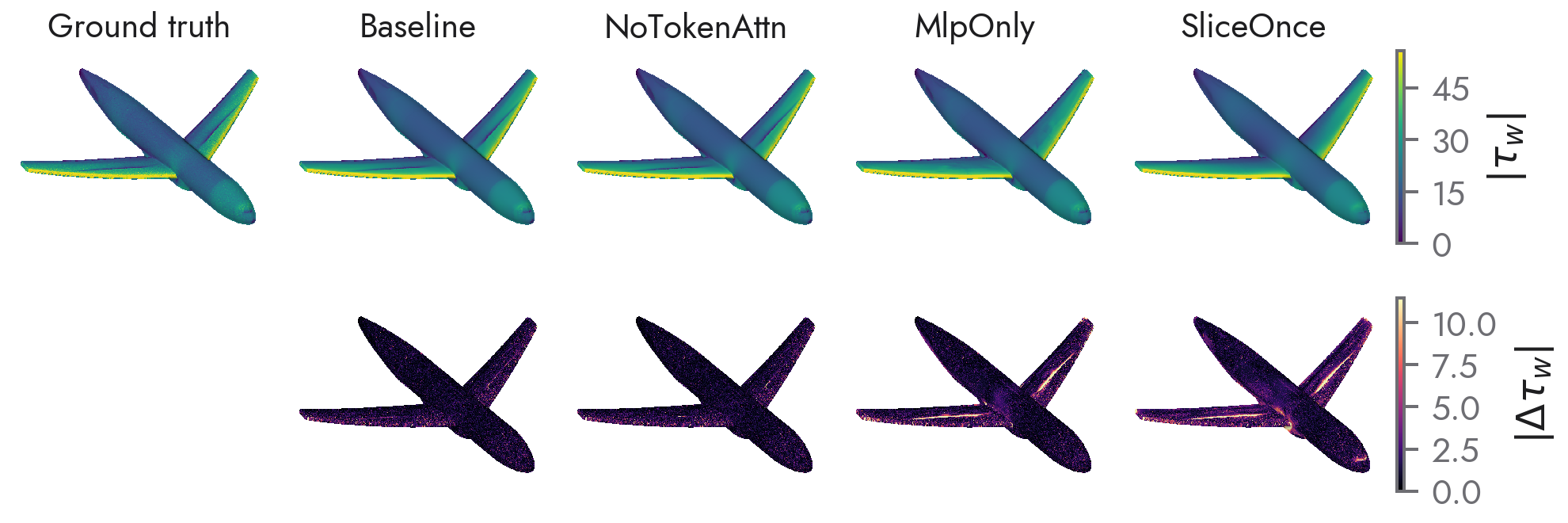}
\caption{\textbf{SHIFT-Wing surface wall shear} for the sample in
Figure~\ref{fig:f-sws}. Relative $L^1$ errors: \baselinev{} $0.057$,
\notoken{} $0.057$, \mlponly{} $0.082$, \sliceonce{} $0.104$.}
\label{fig:f-sws-t}
\end{figure}

\begin{figure}[p]
\centering
\includegraphics[width=\linewidth]{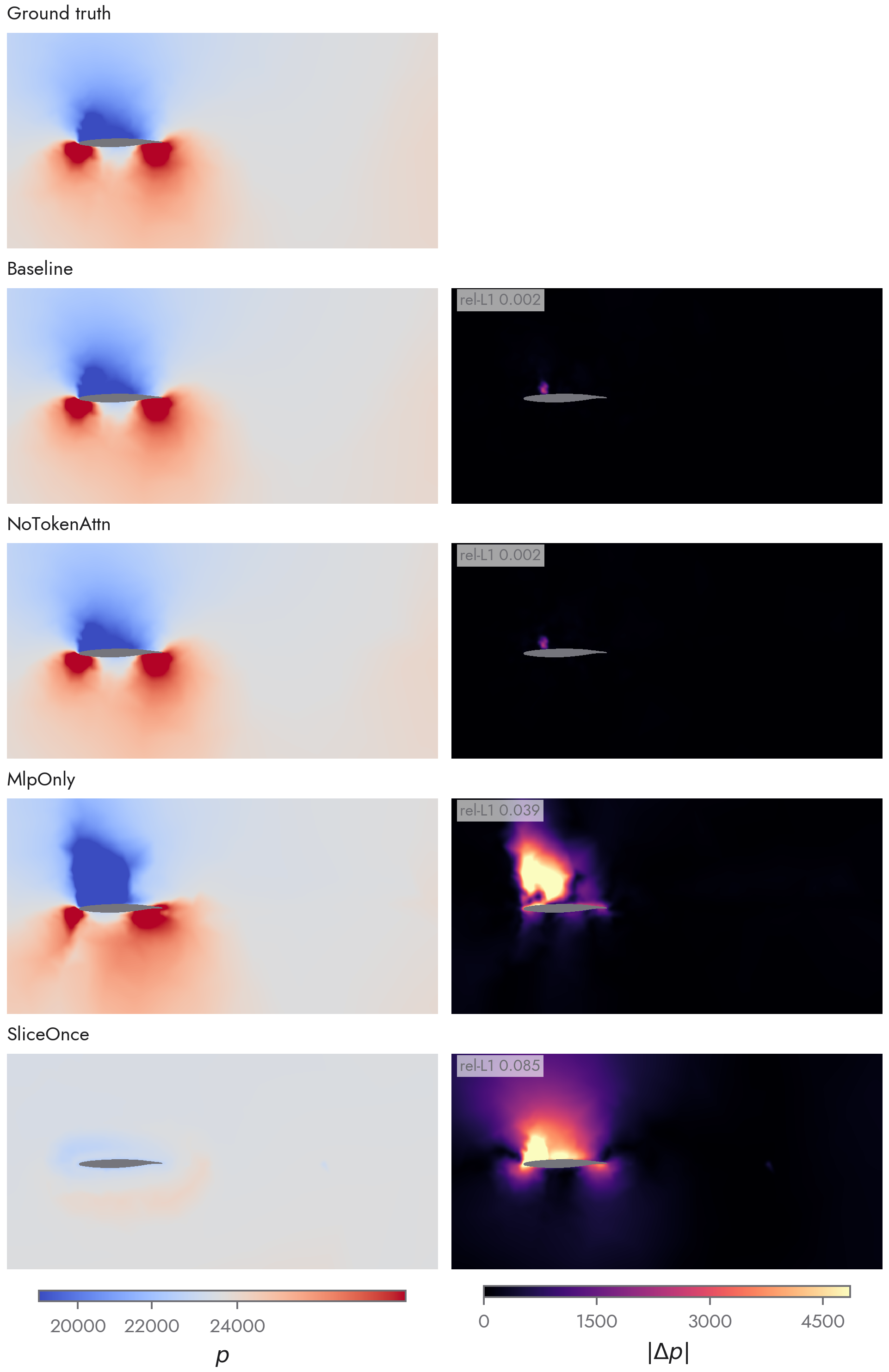}
\caption{\textbf{SHIFT-Wing volume pressure} at $y=15.3$ for sample
\texttt{016032}. Relative $L^1$ errors: \baselinev{} $0.0019$,
\notoken{} $0.0022$, \mlponly{} $0.039$, \sliceonce{} $0.085$.
\sliceonce{} produces a nearly uniform field.}
\label{fig:f-swv}
\end{figure}

\begin{figure}[p]
\centering
\includegraphics[width=\linewidth]{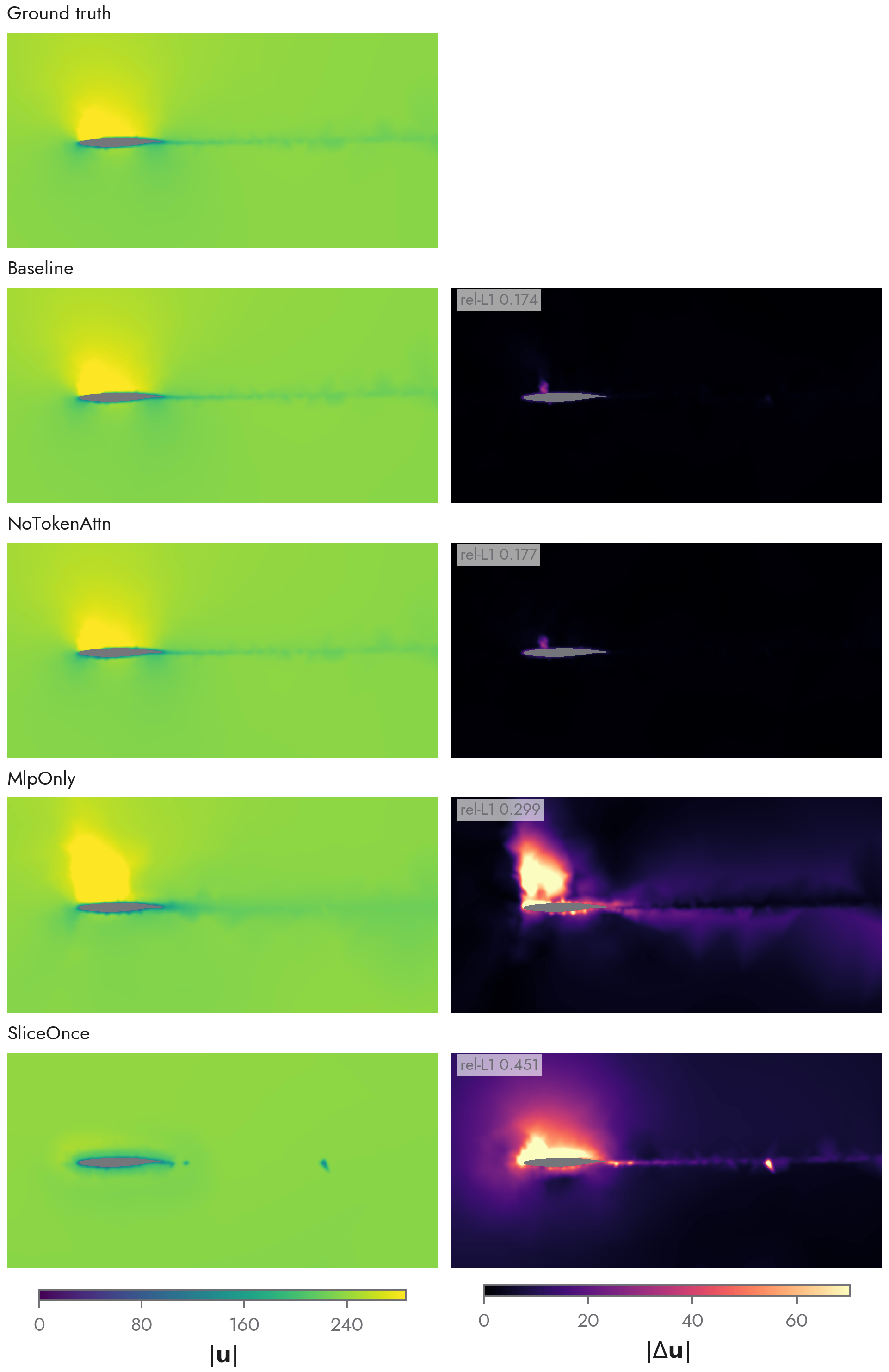}
\caption{\textbf{SHIFT-Wing volume velocity magnitude} on the section
in Figure~\ref{fig:f-swv}. Relative $L^1$ errors: \baselinev{} $0.174$,
\notoken{} $0.177$, \mlponly{} $0.299$, \sliceonce{} $0.451$.}
\label{fig:f-swv-u}
\end{figure}

\clearpage

\begin{figure}[p]
\centering
\includegraphics[width=\linewidth]{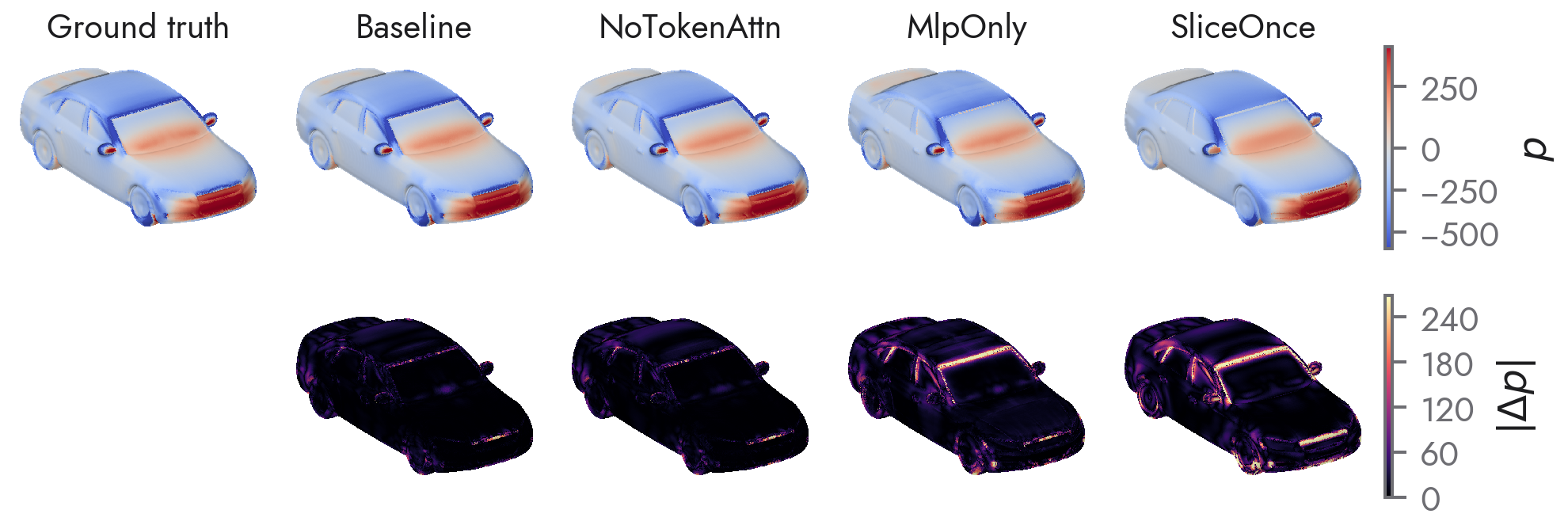}
\caption{\textbf{DrivAerNet++ surface pressure} for car
\texttt{N\_S\_WWC\_WM\_388}. Relative $L^1$ errors: \baselinev{}
$0.095$, \mlponly{} $0.168$, \sliceonce{} $0.210$.}
\label{fig:f-dns}
\end{figure}

\begin{figure}[p]
\centering
\includegraphics[width=\linewidth]{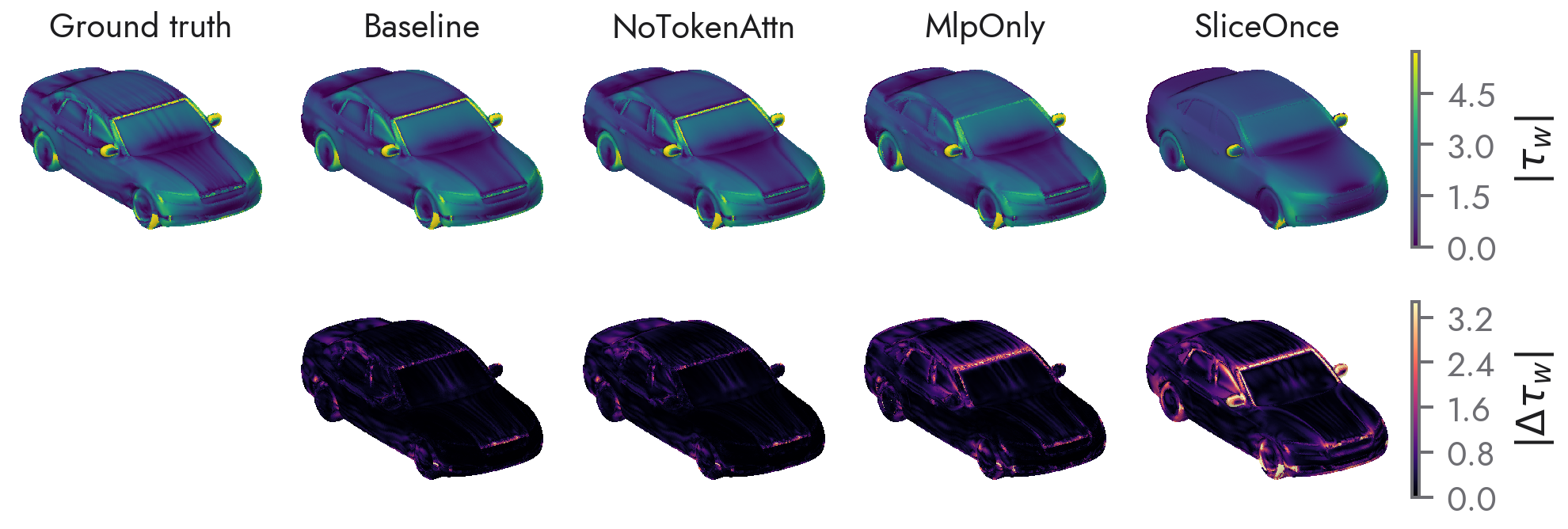}
\caption{\textbf{DrivAerNet++ surface wall shear} for the car in
Figure~\ref{fig:f-dns}. Relative $L^1$ errors: \baselinev{} $0.194$,
\notoken{} $0.204$, \mlponly{} $0.269$, \sliceonce{} $0.476$.}
\label{fig:f-dns-t}
\end{figure}

\begin{figure}[p]
\centering
\includegraphics[width=\linewidth]{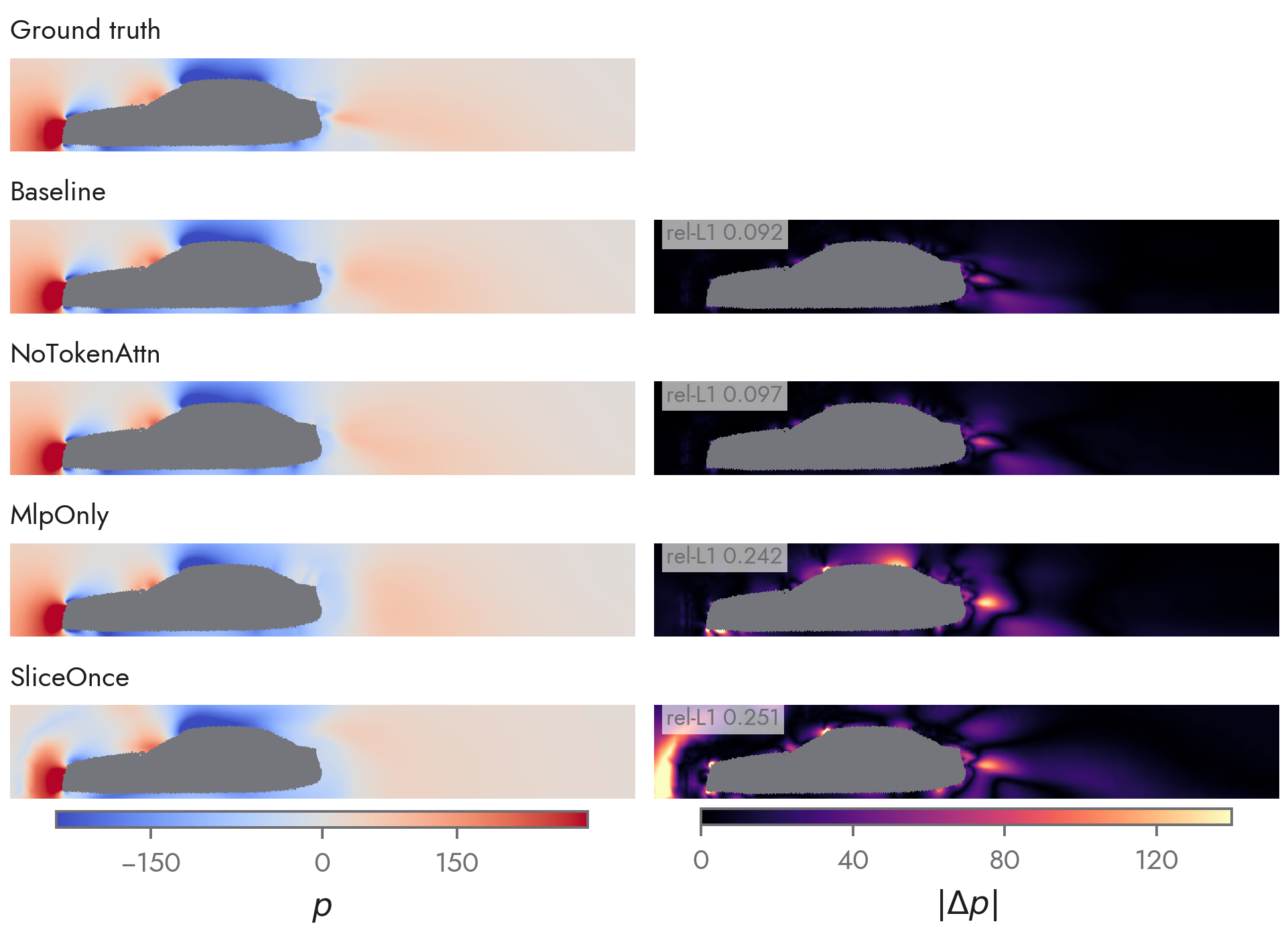}
\caption{\textbf{DrivAerNet++ volume pressure} on the symmetry plane
for car \texttt{N\_S\_WWC\_WM\_391}. Relative $L^1$ errors:
\baselinev{} $0.092$, \notoken{} $0.097$, \mlponly{} $0.242$,
\sliceonce{} $0.251$.}
\label{fig:f-dnv}
\end{figure}

\begin{figure}[p]
\centering
\includegraphics[width=\linewidth]{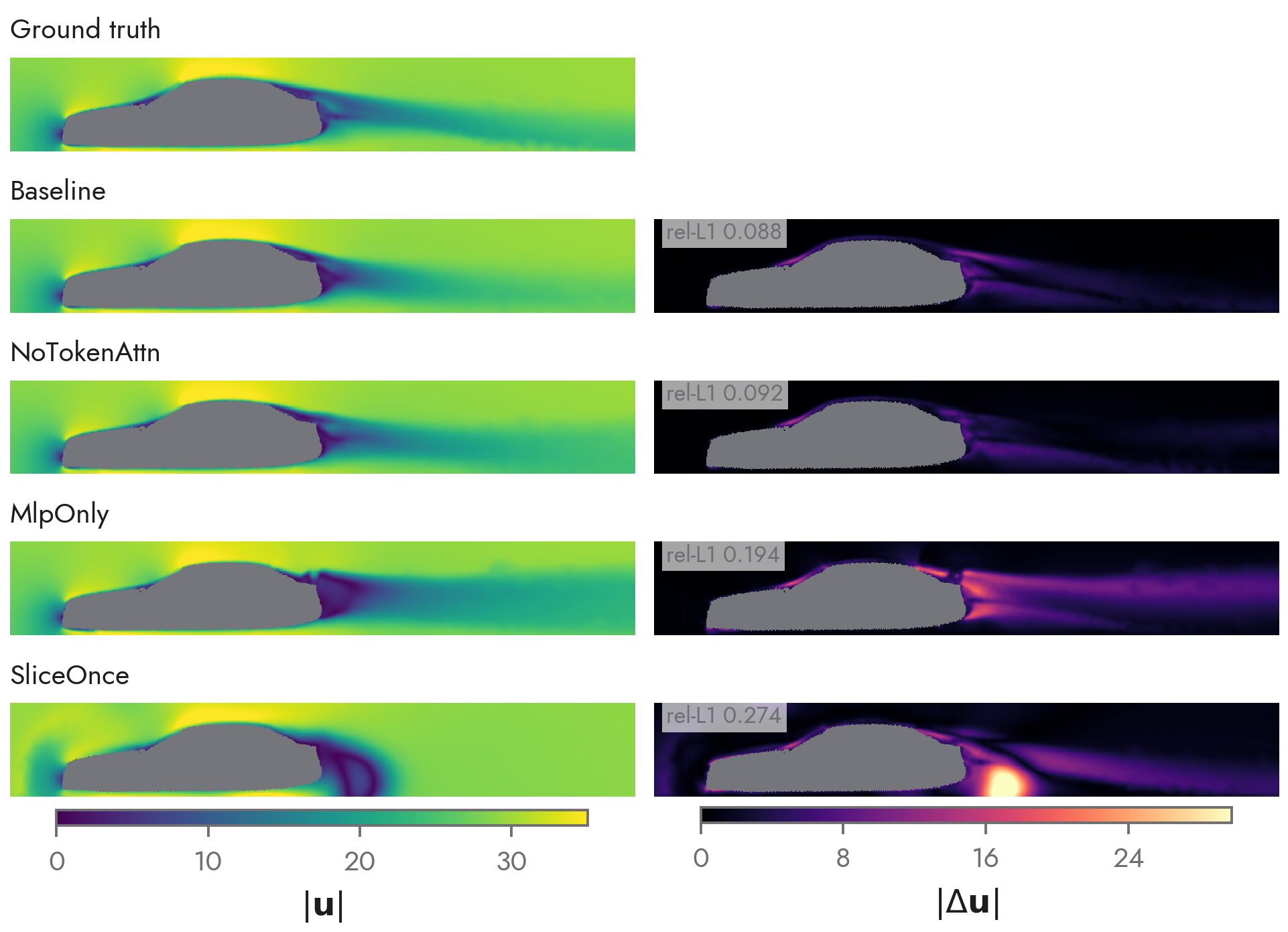}
\caption{\textbf{DrivAerNet++ volume velocity magnitude} on the plane
in Figure~\ref{fig:f-dnv}. Relative $L^1$ errors: \baselinev{} $0.088$,
\notoken{} $0.092$, \mlponly{} $0.194$, \sliceonce{} $0.274$.}
\label{fig:f-dnv-u}
\end{figure}

\begin{figure}[p]
\centering
\includegraphics[width=\linewidth]{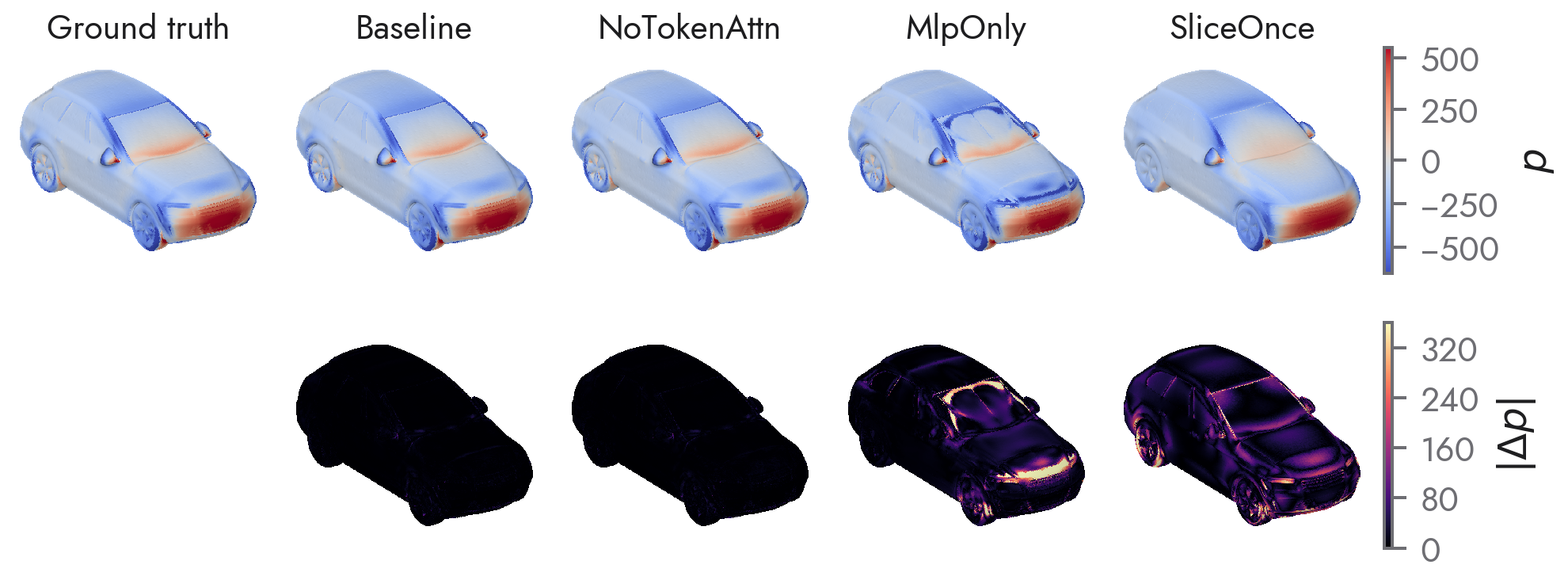}
\caption{\textbf{SHIFT-SUV surface pressure} for car \texttt{run\_00939},
with surface normals as input. Relative $L^1$ errors: \baselinev{}
$0.024$, \notoken{} $0.024$, \mlponly{} $0.113$, \sliceonce{} $0.195$.}
\label{fig:f-suvs}
\end{figure}

\begin{figure}[p]
\centering
\includegraphics[width=\linewidth]{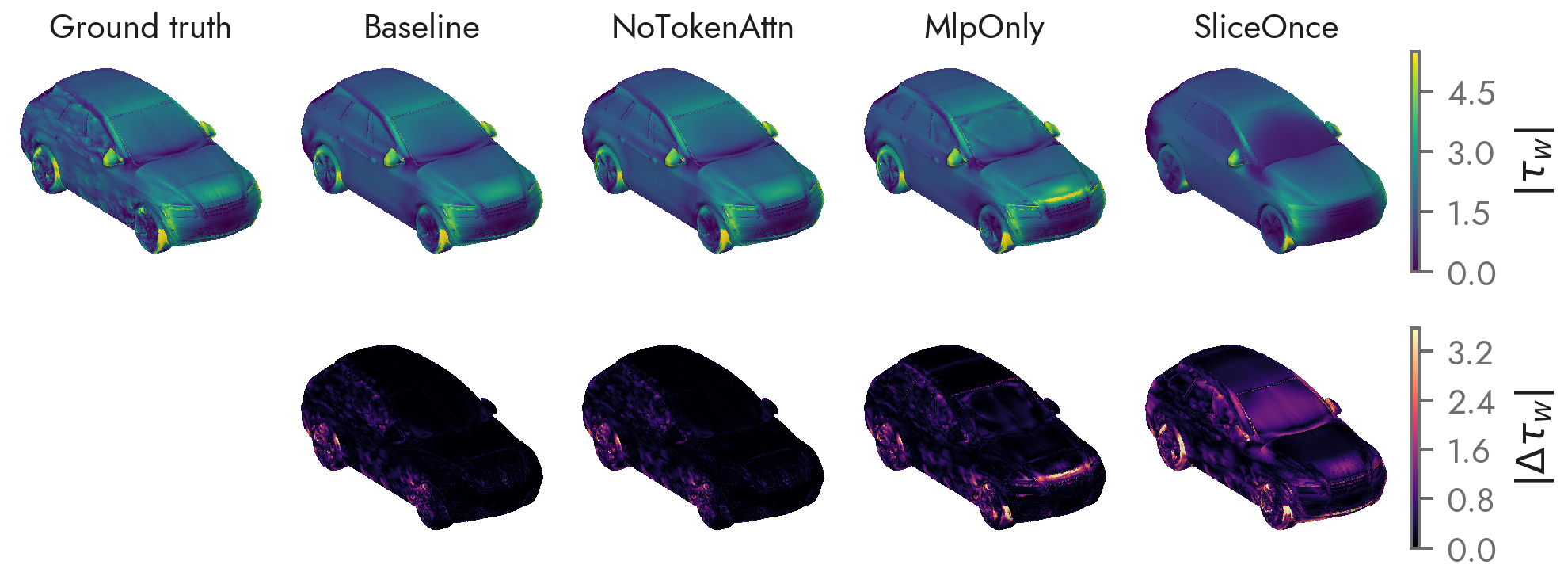}
\caption{\textbf{SHIFT-SUV surface wall shear} for the car in
Figure~\ref{fig:f-suvs}. Relative $L^1$ errors: \baselinev{} $0.256$,
\notoken{} $0.257$, \mlponly{} $0.328$, \sliceonce{} $0.558$.}
\label{fig:f-suvs-t}
\end{figure}

\begin{figure}[p]
\centering
\includegraphics[width=\linewidth]{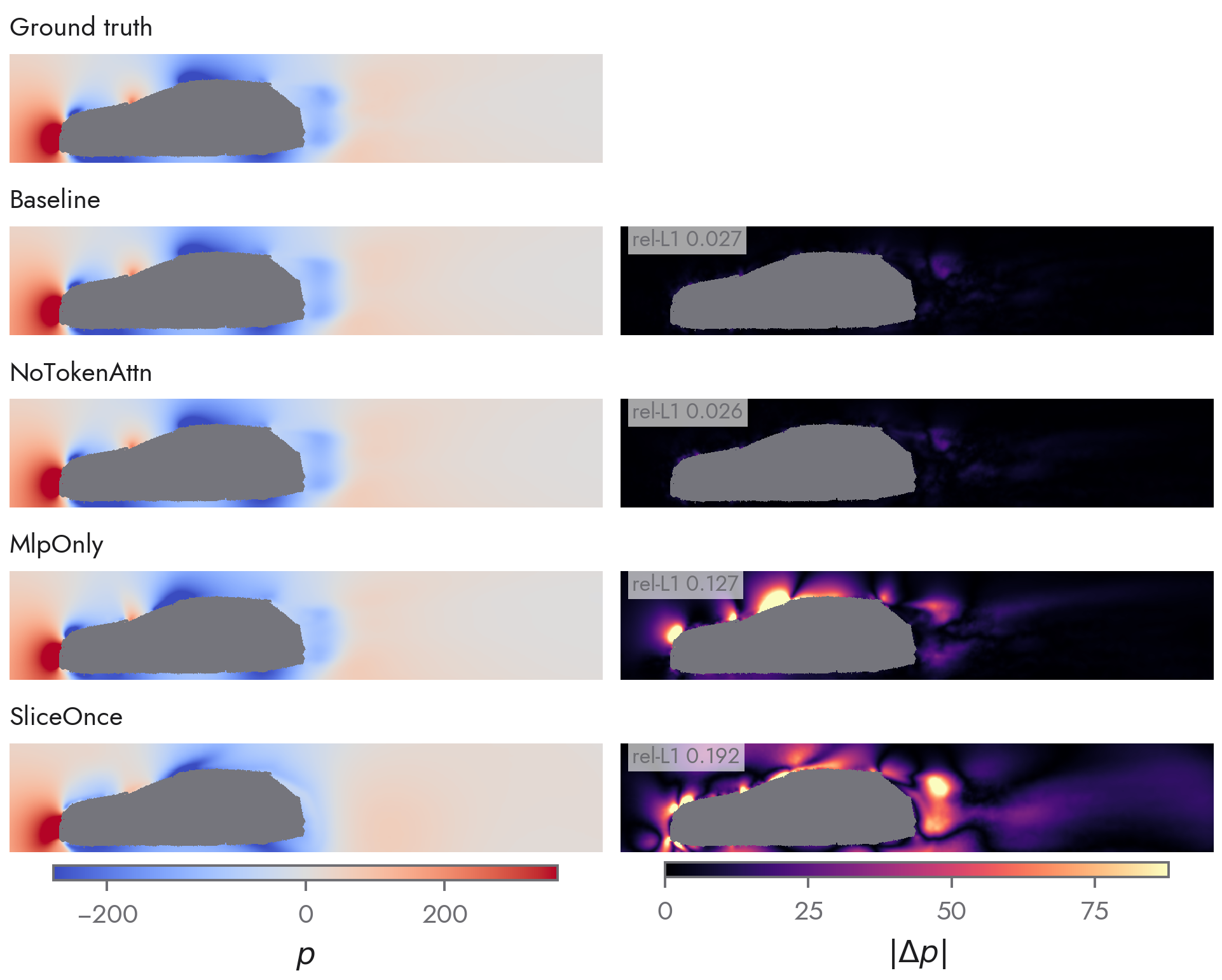}
\caption{\textbf{SHIFT-SUV volume pressure} on the symmetry plane for
car \texttt{run\_00939}, with signed distance as input. Relative $L^1$
errors: \baselinev{} $0.027$, \notoken{} $0.026$, \mlponly{} $0.127$,
\sliceonce{} $0.192$.}
\label{fig:f-suvv}
\end{figure}

\begin{figure}[p]
\centering
\includegraphics[width=\linewidth]{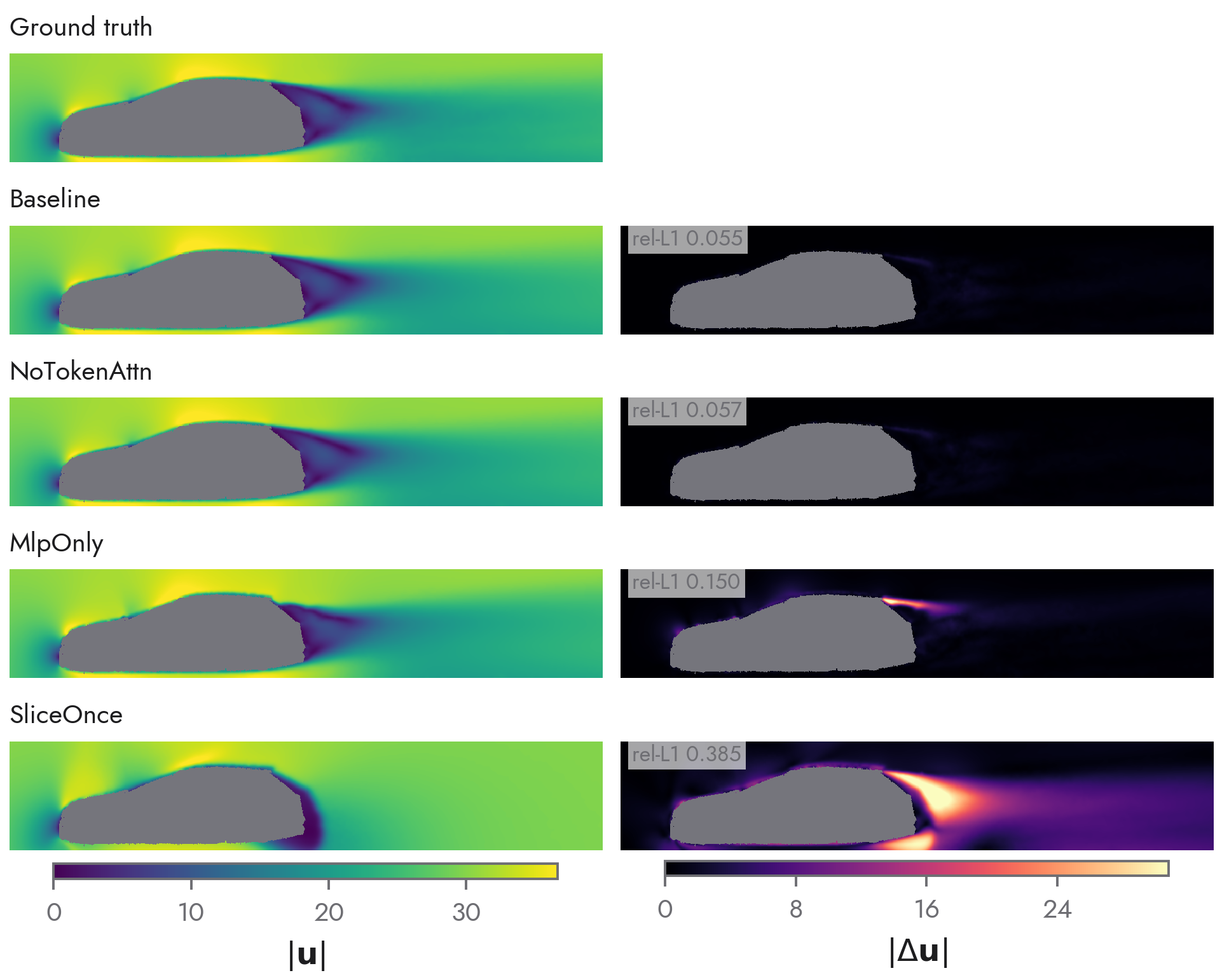}
\caption{\textbf{SHIFT-SUV volume velocity magnitude} on the plane in
Figure~\ref{fig:f-suvv}. Relative $L^1$ errors: \baselinev{} $0.055$,
\notoken{} $0.057$, \mlponly{} $0.150$, \sliceonce{} $0.385$.}
\label{fig:f-suvv-u}
\end{figure}

\begin{figure}[p]
\centering
\includegraphics[width=\linewidth]{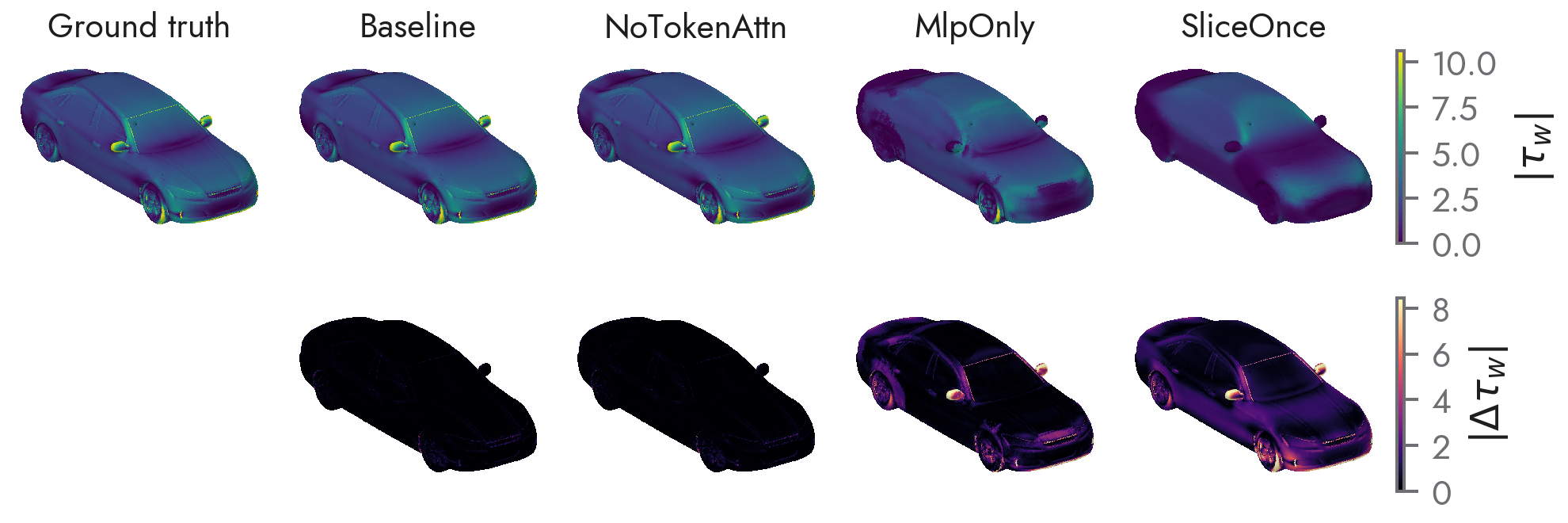}
\caption{\textbf{DrivAerML surface wall shear} for the car in
Figure~\ref{fig:fields}. Relative $L^1$ errors: \baselinev{} $0.107$,
\notoken{} $0.115$, \mlponly{} $0.546$, \sliceonce{} $0.736$.}
\label{fig:f-dmls-t}
\end{figure}

\clearpage

\begin{figure}[p]
\centering
\includegraphics[width=\linewidth]{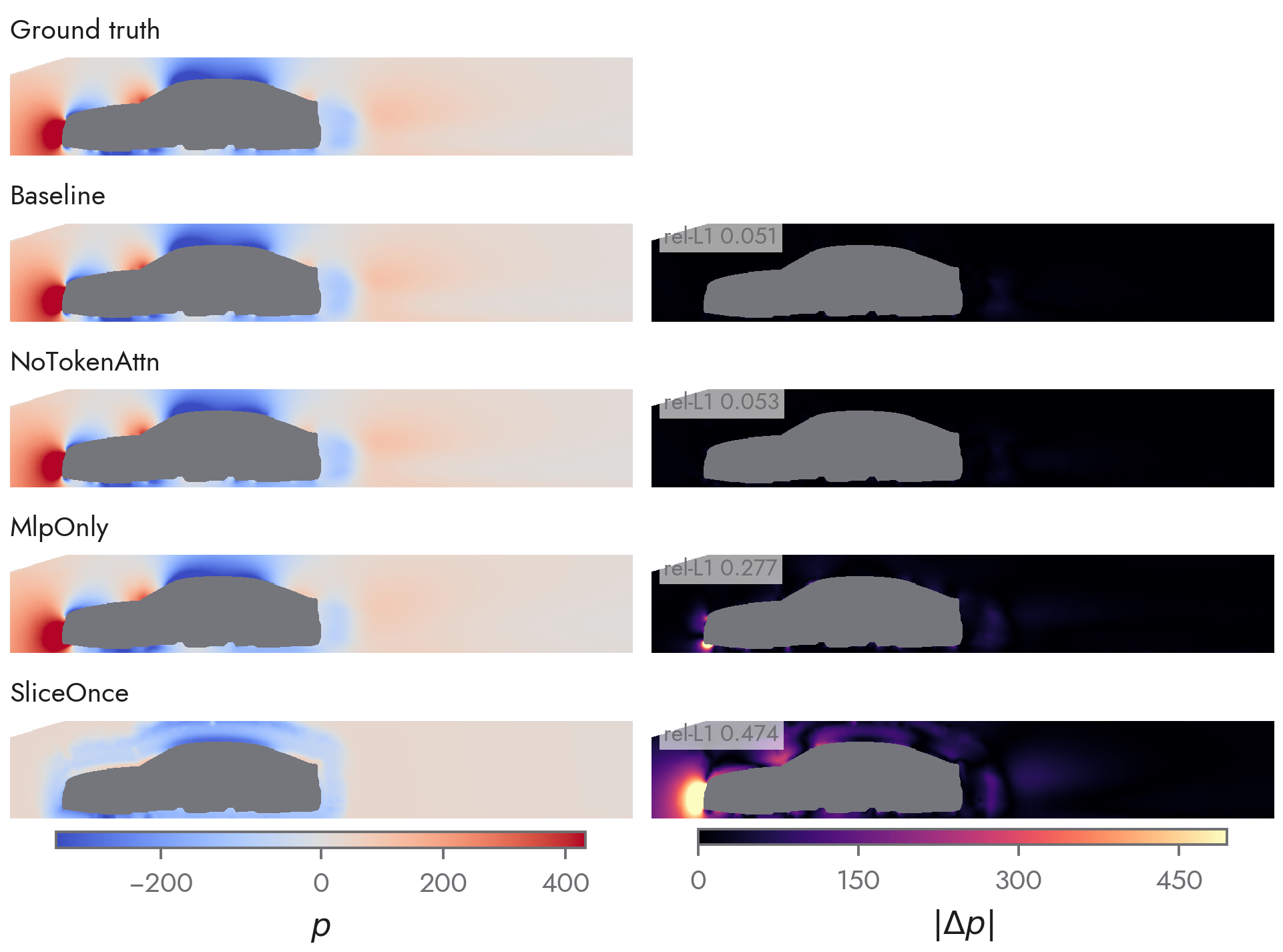}
\caption{\textbf{DrivAerML volume pressure} on the symmetry plane for
car \texttt{run\_108}. Relative $L^1$ errors: \baselinev{} $0.051$,
\notoken{} $0.053$, \mlponly{} $0.277$, \sliceonce{} $0.474$.}
\label{fig:f-dmlv}
\end{figure}

\begin{figure}[p]
\centering
\includegraphics[width=\linewidth]{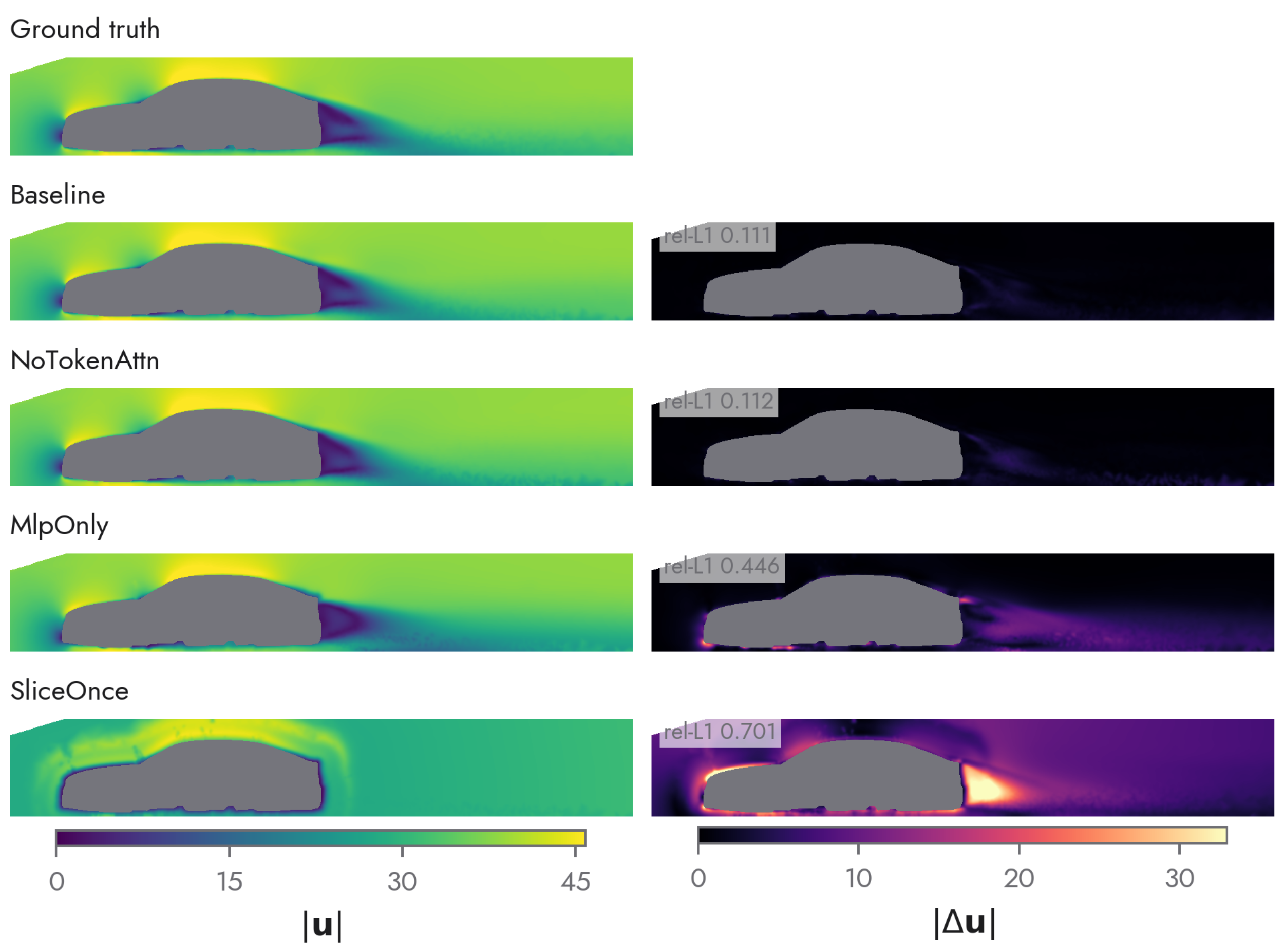}
\caption{\textbf{DrivAerML volume velocity magnitude} on the plane in
Figure~\ref{fig:f-dmlv}. Relative $L^1$ errors: \baselinev{} $0.111$,
\notoken{} $0.112$, \mlponly{} $0.446$, \sliceonce{} $0.701$.}
\label{fig:f-dmlv-u}
\end{figure}

\clearpage

\FloatBarrier

\subsection{Per-variable errors}
\label{app:pervar}

Table~\ref{tab:main} reports a composite relative error on
standardized outputs. Here we evaluate each physical variable
separately after reversing the output normalization.
The resulting relative errors are dimensionless.
Table~\ref{tab:pervar-l1} reports relative $L^1$ errors,
and Table~\ref{tab:pervar-l2} reports relative $L^2$ errors,
with scalar fields evaluated individually and vector fields
as groups. Both metrics support the main comparison:
\notoken{} remains close to the baseline across variables,
whereas \mlponly{} shows larger error increases for pressure
than for velocity or wall shear.

\paragraph{Evaluation.}
We evaluate the test split at the training subsampling ratio,
using the selected checkpoint for each run.
Results use seed 42, except for SHIFT-Wing volume, which uses
seed 43. Taylor--Green is evaluated at full resolution.

\paragraph{Per-variable comparisons.}
Across the aerodynamic benchmarks, the pressure error of
\mlponly{} is $1.8$--$14.8$ times the baseline error,
compared with $1.3$--$5.0$ times for velocity or wall shear.
The largest contrast occurs on SHIFT-Wing volume:
$14.8$ times for pressure and $1.6$ times for velocity.
Taylor--Green shows a smaller effect; unlike the aerodynamic
benchmarks, its inputs include the current field state.

For SHIFT-Wing, pressure errors are computed using dimensional
pressure, whose large ambient offset
($p_\infty \approx 23.8$\,kPa) dominates the denominator.
This makes the relative pressure errors small.
Comparisons between variants on the same benchmark are
therefore more informative than comparisons of absolute
relative errors across benchmarks.

\paragraph{Subsampling variability.}
We repeat each aerodynamic baseline evaluation with six
independent subsample seeds (Table~\ref{tab:noise}).
The relative standard deviation of the per-variable
relative $L^1$ error is at most $0.18\%$.
This indicates low sensitivity to the sampled point set
at the evaluation ratios used here, but does not measure
the difference from full-resolution evaluation.

\begin{table}[htbp]
\caption{\textbf{Subsampling noise of the evaluation metric.} Baseline
checkpoint (seed 43 on SHIFT-Wing volume), test split, six independent
subsample seeds; mean$\pm$std of the per-variable relative $L^{1}$ error.}
\label{tab:noise}
\begin{center}\small
\begin{tabular}{lccc}
\toprule
Benchmark & $p$ & $\mathbf{u}$ / $\boldsymbol{\tau}_w$ & max rel.\ std \\
\midrule
DrivAerNet++ surface & $0.10170\pm.00003$ & $0.19222\pm.00004$ & $0.02\%$ \\
DrivAerNet++ volume  & $0.10634\pm.00004$ & $0.10096\pm.00002$ & $0.04\%$ \\
DrivAerML surface    & $0.03876\pm.00006$ & $0.11286\pm.00009$ & $0.16\%$ \\
DrivAerML volume     & $0.05184\pm.00007$ & $0.11112\pm.00008$ & $0.14\%$ \\
SHIFT-Wing surface   & $0.001018\pm.000001$ & $0.05863\pm.00001$ & $0.10\%$ \\
SHIFT-Wing volume    & $0.002251\pm.000001$ & $0.17355\pm.00002$ & $0.03\%$ \\
SHIFT-SUV surface    & $0.02709\pm.00002$ & $0.26064\pm.00006$ & $0.07\%$ \\
SHIFT-SUV volume     & $0.02986\pm.00006$ & $0.05827\pm.00005$ & $0.18\%$ \\
\bottomrule
\end{tabular}
\end{center}
\end{table}

% ============================================================================
% Appendix E tables -- per-variable test errors in physical units.
% Bodies are GENERATED (do not hand-edit) by
%   cameleon:transolver-studies scripts/eval_pervar/gen_pervar_tables.sh
% from the collector TSV (test split, best checkpoint per Table-1 cell,
% seed 42). Numbers regenerated 2026-08-19 (TEAM rows removed, SHIFT-SUV added).
% ============================================================================
% LaTeX's \input leaves stray tokens behind when it returns at a tabular row
% boundary (\bottomrule then reports "Misplaced \noalign"), so pull the
% generated bodies in with TeX's primitive \input instead.
\makeatletter
\newcommand{\pervarbody}[1]{\input{#1} }
\makeatother

\begin{table}[p]
\caption{\textbf{Per-variable relative $L^{1}$ test error in physical
units.} Best per row in bold; protocol and reading notes in
\S\ref{app:pervar}.}
\label{tab:pervar-l1}
\begin{center}
\scriptsize
\setlength{\tabcolsep}{4pt}
\begin{tabular}{llcccccc}
\toprule
Dataset & Variable & baseline & no-token-attn & MLP-only & untied & frozen & slice-once \\
\midrule
Taylor--Green & $\rho$ & $0.00023$ & $0.00035$ & $0.00024$ & $0.00029$ & $\mathbf{0.00022}$ & $0.00026$ \\
 & $\mathbf{u}$ & $0.1420$ & $\mathbf{0.1397}$ & $0.1474$ & $0.1416$ & $0.1459$ & $0.3581$ \\
\midrule
DrivAerNet++ surf. & $p$ & $\mathbf{0.1017}$ & $0.1057$ & $0.1830$ & $0.1050$ & $0.1047$ & $0.2172$ \\
 & $\boldsymbol{\tau}_w$ & $\mathbf{0.1922}$ & $0.1981$ & $0.2759$ & $0.1928$ & $0.1992$ & $0.4635$ \\
\midrule
DrivAerNet++ vol. & $p$ & $0.1064$ & $\mathbf{0.0953}$ & $0.2188$ & $0.0979$ & $0.0963$ & $0.2253$ \\
 & $\mathbf{u}$ & $0.1010$ & $0.0940$ & $0.1639$ & $0.0934$ & $\mathbf{0.0933}$ & $0.2418$ \\
\midrule
DrivAerML surf. & $p$ & $0.0387$ & $0.0417$ & $0.2622$ & $\mathbf{0.0373}$ & $0.0440$ & $0.2592$ \\
 & $\boldsymbol{\tau}_w$ & $0.1128$ & $0.1224$ & $0.5645$ & $\mathbf{0.1088}$ & $0.1325$ & $0.7440$ \\
\midrule
DrivAerML vol. & $p$ & $0.0517$ & $0.0512$ & $0.2885$ & $\mathbf{0.0509}$ & $0.0573$ & $0.4612$ \\
 & $\mathbf{u}$ & $0.1112$ & $0.1089$ & $0.4470$ & $\mathbf{0.1083}$ & $0.1242$ & $0.6873$ \\
\midrule
SHIFT-Wing surf. & $p$ & $0.00102$ & $0.00104$ & $0.0107$ & $\mathbf{0.00099}$ & $0.00100$ & $0.00849$ \\
 & $\boldsymbol{\tau}_w$ & $0.0586$ & $0.0586$ & $0.0870$ & $\mathbf{0.0584}$ & $0.0584$ & $0.1125$ \\
\midrule
SHIFT-Wing vol. & $p$ & $0.00225$ & $0.00222$ & $0.0333$ & $\mathbf{0.00215}$ & $0.00225$ & $0.0766$ \\
 & $\mathbf{u}$ & $\mathbf{0.1735}$ & $0.1755$ & $0.2843$ & $0.1758$ & $0.1802$ & $0.4471$ \\
\midrule
SHIFT-SUV surf. & $p$ & $\mathbf{0.0271}$ & $0.0280$ & $0.1131$ & $0.0271$ & $0.0286$ & $0.2043$ \\
 & $\boldsymbol{\tau}_w$ & $\mathbf{0.2607}$ & $0.2618$ & $0.3296$ & $0.2608$ & $0.2637$ & $0.5728$ \\
\midrule
SHIFT-SUV vol. & $p$ & $0.0298$ & $0.0305$ & $0.1330$ & $\mathbf{0.0279}$ & $0.0328$ & $0.2035$ \\
 & $\mathbf{u}$ & $0.0582$ & $0.0595$ & $0.1537$ & $\mathbf{0.0550}$ & $0.0647$ & $0.3843$ \\

\bottomrule
\end{tabular}
\end{center}
\end{table}

\begin{table}[p]
\caption{\textbf{Per-variable relative $L^{2}$ test error in physical
units} (same protocol as Table~\ref{tab:pervar-l1}).}
\label{tab:pervar-l2}
\begin{center}
\scriptsize
\setlength{\tabcolsep}{4pt}
\begin{tabular}{llcccccc}
\toprule
Dataset & Variable & baseline & no-token-attn & MLP-only & untied & frozen & slice-once \\
\midrule
Taylor--Green & $\rho$ & $0.00034$ & $0.00044$ & $0.00035$ & $0.00039$ & $\mathbf{0.00031}$ & $0.00054$ \\
 & $\mathbf{u}$ & $0.1850$ & $\mathbf{0.1826}$ & $0.1906$ & $0.1845$ & $0.1888$ & $0.4470$ \\
\midrule
DrivAerNet++ surf. & $p$ & $\mathbf{0.1509}$ & $0.1563$ & $0.2510$ & $0.1553$ & $0.1564$ & $0.3479$ \\
 & $\boldsymbol{\tau}_w$ & $0.2265$ & $0.2316$ & $0.3098$ & $\mathbf{0.2262}$ & $0.2332$ & $0.5272$ \\
\midrule
DrivAerNet++ vol. & $p$ & $0.1621$ & $\mathbf{0.1462}$ & $0.3087$ & $0.1481$ & $0.1474$ & $0.3709$ \\
 & $\mathbf{u}$ & $0.1337$ & $0.1234$ & $0.2088$ & $0.1229$ & $\mathbf{0.1222}$ & $0.2770$ \\
\midrule
DrivAerML surf. & $p$ & $0.0614$ & $0.0656$ & $0.4053$ & $\mathbf{0.0590}$ & $0.0711$ & $0.3840$ \\
 & $\boldsymbol{\tau}_w$ & $0.1204$ & $0.1294$ & $0.6067$ & $\mathbf{0.1163}$ & $0.1385$ & $0.7404$ \\
\midrule
DrivAerML vol. & $p$ & $0.0729$ & $\mathbf{0.0725}$ & $0.4265$ & $0.0726$ & $0.0837$ & $0.6132$ \\
 & $\mathbf{u}$ & $0.1195$ & $\mathbf{0.1177}$ & $0.4749$ & $0.1178$ & $0.1343$ & $0.6550$ \\
\midrule
SHIFT-Wing surf. & $p$ & $0.00356$ & $0.00361$ & $0.0288$ & $0.00352$ & $\mathbf{0.00330}$ & $0.0191$ \\
 & $\boldsymbol{\tau}_w$ & $0.0793$ & $0.0791$ & $0.1202$ & $0.0790$ & $\mathbf{0.0788}$ & $0.1453$ \\
\midrule
SHIFT-Wing vol. & $p$ & $0.0101$ & $\mathbf{0.00939}$ & $0.0782$ & $0.00971$ & $0.0102$ & $0.1318$ \\
 & $\mathbf{u}$ & $\mathbf{0.2041}$ & $0.2069$ & $0.3239$ & $0.2078$ & $0.2124$ & $0.4490$ \\
\midrule
SHIFT-SUV surf. & $p$ & $0.0444$ & $0.0461$ & $0.1859$ & $\mathbf{0.0442}$ & $0.0465$ & $0.3757$ \\
 & $\boldsymbol{\tau}_w$ & $\mathbf{0.2933}$ & $0.2946$ & $0.3649$ & $0.2937$ & $0.2954$ & $0.5754$ \\
\midrule
SHIFT-SUV vol. & $p$ & $0.0537$ & $0.0546$ & $0.2181$ & $\mathbf{0.0501}$ & $0.0595$ & $0.3436$ \\
 & $\mathbf{u}$ & $0.0729$ & $0.0738$ & $0.2049$ & $\mathbf{0.0685}$ & $0.0801$ & $0.4039$ \\

\bottomrule
\end{tabular}
\end{center}
\end{table}

\section{LinearNO implementation and latent width}
\label{app:linearno}
For the \untiedln{} variant in Figure~\ref{fig:controls}(a),
we use the released layer implementation of
\citet{hu2026linearno} within the same model framework as
Transolver. The lifting MLP, pre-layer normalization,
residual connections, and output head are unchanged, so
the comparison isolates the replacement of the attention
layer. This experiment uses one seed on SHIFT-SUV surface.

\paragraph{Normalization.}
Both layers aggregate point features into tokens and map
the tokens back to points, but construct their aggregation
weights differently (Figure~\ref{fig:linearno}(a)).
Transolver applies a softmax over the $G$ slices at each
point and uses the resulting weights in both directions.
Each aggregated token is then divided by the sum of its
weights over points.
LinearNO uses separate weights for slicing and deslicing.
Its deslice weights $q$ are normalized over latent channels,
whereas its slice weights $k$ are normalized over points.
The released implementation explicitly stores both tensors,
each of shape $[B,H,N,M]$, where $M$ is the latent width
per head. Their memory footprint therefore grows linearly
with $M$. For Transolver, \flashslice{} avoids materializing
the corresponding point-to-slice weights
(\S\ref{sec:flashslice}).

\begin{figure}[t]
\centering
\includegraphics[width=\linewidth]{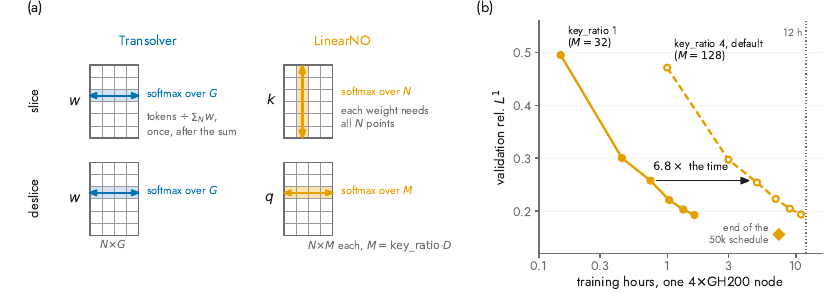}
\caption{\textbf{LinearNO normalization and latent width.}
(a) Slice and deslice weights in Transolver and LinearNO,with their normalization axes. LinearNO uses latent width
$M=\mathrm{key\_ratio}\times d_{\mathrm{head}}$.
(b) Validation errors on SHIFT-SUV surface at shared training
steps, plotted against runtime estimated from measured
iteration times on a four-GH200 node.
The settings $\mathrm{key\_ratio}=1$ and $4$ give $M=32$
and $128$, respectively.
The diamond marks the completion of $50$k steps for
$\mathrm{key\_ratio}=1$.}
\label{fig:linearno}
\end{figure}

\paragraph{Latent width.}
LinearNO sets
$M=\mathrm{key\_ratio}\times d_{\mathrm{head}}$.
With our head dimension $d_{\mathrm{head}}=32$, the published
default $\mathrm{key\_ratio}=4$ gives $M=128$, compared with
Transolver's $G=32$ slices.
We use $\mathrm{key\_ratio}=1$ in the main comparison to
match these widths and run the default setting as an
additional comparison.

\paragraph{Runtime and validation error.}
On SHIFT-SUV surface, using one node with four GH200 GPUs,
$\mathrm{key\_ratio}=1$ runs at $1.88$ iterations per second
and completes $50$k steps in $7$\,h\,$25$\,m.
The default setting runs at $3.59$ seconds per iteration,
approximately $6.8$ times slower.
At that rate, $50$k steps would take approximately $50$ hours.
These timings describe the implementations evaluated here.

The default-width run reaches $11{,}531$ steps within one
12-hour allocation. Both runs use the same $50$k-step
cosine schedule, allowing comparison at the same steps
and learning rates.
Across eleven shared evaluation points, the ratio of
validation error at $\mathrm{key\_ratio}=4$ to that at
$\mathrm{key\_ratio}=1$ averages $0.997$, with a range
of $0.95$--$1.02$.
Table~\ref{tab:linearno-keyratio} lists selected checkpoints.

The two widths therefore give similar validation errors
over the observed training interval, while the narrower
setting is substantially faster.
We use $\mathrm{key\_ratio}=1$ for the width-matched
comparison; the partial default-width run does not
establish how the two settings compare after $50$k steps.

\begin{table}[htbp]
\caption{\textbf{LinearNO latent width on SHIFT-SUV surface.}
Validation relative $L^1$ errors at selected shared checkpoints.
Both runs use a $50$k-step cosine schedule.
The ratio is the error at $\mathrm{key\_ratio}=4$ divided
by that at $\mathrm{key\_ratio}=1$.}
\label{tab:linearno-keyratio}
\begin{center}\small
\begin{tabular}{lcccccc}
\toprule
step & 1k & 3k & 5k & 7k & 9k & 11k \\
\midrule
$\mathrm{key\_ratio}=1$ & $0.4952$ & $0.3005$ & $0.2577$ & $0.2210$ & $0.2032$ & $0.1927$ \\
$\mathrm{key\_ratio}=4$ & $0.4714$ & $0.2973$ & $0.2543$ & $0.2232$ & $0.2049$ & $0.1938$ \\
ratio & $0.95$ & $0.99$ & $0.99$ & $1.01$ & $1.01$ & $1.01$ \\
\bottomrule
\end{tabular}
\end{center}
\end{table}

\paragraph{A second benchmark, and the cost of the layer.}
Table~\ref{tab:rb-factorial} extends the $2\times2$ grid of token attention
$\times$ projection tying of Figure~\ref{fig:controls}(a) to SHIFT-Wing
surface. On both benchmarks every cell lies within $0.7\%$ of the tied
\baselinev{}, so the attention is removable with tied projections as well as
with untied ones. Table~\ref{tab:rb-cost} reports what the LinearNO layer
costs. Over whole training runs it takes $2.4\times$ the wall clock of
Transolver at $\mathrm{key\_ratio}=1$ and about $16\times$ at the published
$\mathrm{key\_ratio}=4$; untying alone (\untied{}) costs $9\%$ on SHIFT-SUV
surface ($3.4$\,h against $3.1$\,h). A single training step shows the same
ordering at every size, and at its published width the LinearNO layer runs out
of memory at $524$k points.

\begin{table}[htbp]
\caption{\textbf{Token attention $\times$ projection tying on two benchmarks.} Best
validation relative $L^{1}$ and change against the tied \baselinev{} of the same row.
\untiedln{} is the LinearNO layer (untied, attention-free, slice weights normalized over
the points).}
\label{tab:rb-factorial}
\begin{center}\footnotesize
\setlength{\tabcolsep}{4pt}
\begin{tabular}{lcccc}
\toprule
 & \baselinev{} & \notoken{} & \untied{} & \untiedln{} \\
 & tied, attention & tied, none & untied, attention & untied, none \\
\midrule
SHIFT-SUV surface  & $0.1565$ & $0.1574$\,($+0.5\%$) & $0.1558$\,($-0.5\%$) & $0.1560$\,($-0.4\%$) \\
SHIFT-Wing surface & $0.0580$ & $0.0579$\,($-0.2\%$) & $0.0576$\,($-0.7\%$) & $0.0576$\,($-0.7\%$) \\
\bottomrule
\end{tabular}
\end{center}
\end{table}

\begin{table}[htbp]
\caption{\textbf{What the LinearNO layer costs.} (a)~Wall clock of whole training runs on
one four-GPU GH200 node, same pipeline per benchmark. (b)~One training step of the full model
($L{=}8$, $C{=}256$, $H{=}8$, $G{=}32$ or key ratio $1$/$4$), one GH200, fp32 with TF32
disabled: time and peak memory. Absolute times are therefore not comparable with
Table~\ref{tab:sys}; the columns here are like for like.}
\label{tab:rb-cost}
\begin{center}\footnotesize
\setlength{\tabcolsep}{4pt}
\begin{tabular}{lccc}
\multicolumn{4}{l}{(a) training wall clock} \\
\toprule
benchmark (steps) & Transolver & LinearNO, key ratio 1 & LinearNO, key ratio 4 (default) \\
\midrule
SHIFT-SUV surface (50k)  & $3.1$\,h  & $7.4$\,h ($2.37\times$)  & ${\approx}50$\,h (${\approx}16\times$)$^{\ast}$ \\
SHIFT-Wing surface (250k) & $11.7$\,h & $27.6$\,h ($2.36\times$) & -- \\
\bottomrule
\multicolumn{4}{l}{\scriptsize $^{\ast}$projected from $11.5$k steps in a $12$\,h allocation (Table~\ref{tab:linearno-keyratio}).}\\
\end{tabular}

\vspace{6pt}
\begin{tabular}{lcccc}
\multicolumn{5}{l}{(b) one training step, ms / GB} \\
\toprule
$N$ & Transolver & \notoken{} $+$ \flashslice{} & LinearNO, key ratio 1 & LinearNO, key ratio 4 \\
\midrule
$32$k  & $53$ / $3.6$   & $32$ / $3.1$   & $94$ / $5.1$    & $526$ / $8.3$ \\
$131$k & $198$ / $14.1$ & $117$ / $12.1$ & $363$ / $20.1$  & $2079$ / $33.1$ \\
$262$k & $386$ / $28.0$ & $226$ / $24.0$ & $718$ / $40.0$  & $4160$ / $66.1$ \\
$524$k & $782$ / $55.8$ & $446$ / $47.8$ & $1437$ / $79.8$ & OOM \\
\bottomrule
\end{tabular}
\end{center}
\end{table}

\section{\flashslice{}: implementation and evaluation}
\label{app:systems}

This appendix describes the kernels, numerical validation, and
performance measurements supporting \S\ref{sec:flashslice}.
Experiments use a node with four NVIDIA GH200 GPUs
($95$\,GB usable HBM per GPU), PyTorch~2.5, and Triton~3.0
\citep{ansel2024pytorch2,tillet2019triton}.
Unless stated otherwise, measurements use one GPU with
$C=256$, $H=8$, $G=32$, $L=8$, and batch size $1$.

Each comparison is measured within the same job using fixed
seeds and at least ten timed iterations after warm-up.
We report median times; observed min--max spreads are below
$2\%$. Point counts are rounded in the tables. The main sweep
uses powers of two from $2^{18}$ to $2^{24}$, with additional
intermediate sizes near the memory limit.

\subsection{Kernel design}
\paragraph{Single-tile kernels.}
For each head, the slice operation computes
\[
w_{ig}
=
\operatorname{softmax}_{G}
\left(
\frac{x_{\mathrm{mid},i}W^\top+b}{\tau}
\right)_g,
\qquad
z_g^{\mathrm{num}}=\sum_i w_{ig}f_i,
\qquad
s_g=\sum_i w_{ig}.
\]
The normalized tokens are
$z_g=z_g^{\mathrm{num}}/(s_g+\varepsilon)$.
Deslicing recomputes the weights and maps the mixed tokens
back to points.

When the slice axis fits within a tile, the softmax is
computed locally for each point. The weights remain in
registers rather than being written to HBM, and are
recomputed during backward.
Reductions use fp32 partial buffers followed by a
deterministic tree reduction without atomics.
The kernels read the projections in their native
$(N,H,D)$ layout, avoiding the eager implementation's
layout copies.

Input projections remain separate GEMMs, and the token
transformation remains in PyTorch.
In a forward-only comparison over 36 tile configurations,
fusing the projections achieved only
$0.18$--$0.75$ times the throughput of separate cuBLAS
projections followed by the kernel.

The single-tile implementation supports power-of-two
$D$ and $G$ in $[16,128]$, with value width equal to $D$.
Other supported shapes use the $G$-blocked kernels;
$D>256$ falls back to eager execution with a warning.
Tile settings are tuned separately for each $G$.
The forward and backward kernels are registered as
PyTorch custom operators with shape inference and
autograd support, allowing the tested model to compile
as a single graph.

\paragraph{$G$-blocked kernels.}
For larger or irregular slice counts, the kernels process
the slice axis in blocks of $G_B$.
A statistics pass computes the maximum logit $m_n$ and
the shifted exponential sum $l_n$ for each point and head.
Subsequent kernels reconstruct weights blockwise as
\[
w_{ng}
=
\frac{\exp(a_{ng}-m_n)}{l_n},
\]
where $a_{ng}$ is the temperature-scaled logit.

Slice programs own a block of slots and reduce over points,
producing partial token numerators and weight sums.
Column normalization is applied after these partials have
been combined.
Deslice programs own a block of points and reduce over slots.
When statistics are unavailable, deslicing uses an online
softmax reduction and saves $(m,l)$ for reuse.
Operations with the same membership weights can share
these statistics.

Backward uses separate point-owned and slot-owned kernels.
The point-owned kernel makes two passes over slots to
compute the softmax correction and input gradients.
The slot-owned kernel accumulates projection, bias, and
token gradients over points.
This requires more logit recomputation than the single-tile
implementation. The blocked family therefore uses its own
tile settings, with $G_B$ selected from $16$--$64$ according
to the dot-product tile size.

\subsection{Numerical validation}

\paragraph{Precision modes and acceptance criteria.}
Table~\ref{tab:dotmodes} lists the supported dot-product
modes. All modes accumulate in fp32.
The default \texttt{ieee} mode uses fp32 FMA operations;
the other modes trade numerical accuracy for tensor-core
throughput.

We compare outputs and parameter gradients against an
fp64 eager reference, using the corresponding eager error
to set the acceptance threshold.
For the default fp32 mode, the threshold is
$1.25$ times the eager fp32 error.
It is relaxed to $3$ times when the eager reference error
exceeds $10^{-4}$.
The bf16-dot modes use a threshold of $2$ times the
eager-autocast error, and the TF32 modes use $10$ times
the eager fp32 error.
For the blocked kernels, the operator-level temperature
gradient uses a $2$-times threshold; the layer-level tests
at the paper configurations retain the $1.25$-times
threshold.
Repeated fused runs must also agree bitwise in these tests.

\begin{table}[htbp]
\caption{\textbf{Dot-precision levels of the kernels.} Accumulation is
fp32 throughout; the gate is the admissible error against an fp64
reference, relative to the corresponding eager computation. \texttt{ieee}
is Triton's FMA path and needs no tensor cores.}
\label{tab:dotmodes}
\begin{center}\footnotesize
\resizebox{\linewidth}{!}{%
\begin{tabular}{lllll}
\toprule
level & logits dot & value dots & inputs & gate \\
\midrule
\texttt{ieee} (default) & fp32 FMA & fp32 FMA & fp32 or bf16 & $1.25\times$ eager fp32 \\
\texttt{tf32}   & fp32 FMA & tf32 tensor cores & fp32 or bf16 & $10\times$ \\
\texttt{bf16v}  & fp32 FMA & bf16 tensor cores & bf16 & $2\times$ eager bf16 autocast \\
\texttt{bf16}   & bf16 tensor cores & bf16 tensor cores & bf16 & $2\times$ eager bf16 autocast \\
\texttt{tf32x3} & three tf32 products & three tf32 products & fp32 & $10\times$; measured $3$--$7\times$ eager fp32 \\
\bottomrule
\end{tabular}}
\end{center}
\end{table}

\paragraph{Temperature gradients.}
The temperature gradient is particularly sensitive to
rounding in the softmax backward:
\[
\frac{\partial L}{\partial\tau}
=
-\sum_{n,g}
\frac{\partial L}{\partial a_{ng}}
\frac{a_{ng}}{\tau}.
\]
In exact arithmetic,
$\sum_g \partial L/\partial a_{ng}=0$ for each point.
Rounding errors in this cancellation can affect the
temperature gradient.

We therefore save $(m,l)$ directly, construct the softmax
correction and logit gradients from the same rounded
products, and use correctly rounded reciprocals for
$1/l_n$ and $1/\tau$.
In our tests, approximate division increased the temperature
gradient error to $8$--$14$ times the eager error.

\paragraph{Layer and optimizer tests.}
With TF32 disabled, the tested single-tile configurations
and ablation variants have fused output and parameter-gradient
errors at or below those of eager fp32, typically
$5\times10^{-7}$ compared with $1$--$4\times10^{-6}$.
Starting from identical parameters, a 60-step optimizer
comparison gives a relative loss difference of at most
$2\times10^{-6}$.

Under bf16 autocast, the output errors are
$2.5\times10^{-3}$ for fused and $2.6\times10^{-3}$
for eager, with gradient errors within $5\%$ of each other.
Table~\ref{tab:sys} uses fp32 dots for the fp32 measurements
and bf16 tensor-core dots for the autocast measurements.

\subsection{Single-GPU performance}

\paragraph{Point-count sweep.}
Table~\ref{tab:sys-eager-full} reports inference and training
measurements against uncompiled eager PyTorch, without
activation checkpointing.
The fused implementation reduces time and peak memory by
avoiding materialized slice weights and layout copies.

Additional measurements near the memory limit increase
the largest tested inference size from $8.4$M to $10.5$M
points in fp32 and from $10.5$M to $14.7$M under bf16
autocast.
For training, the corresponding sizes increase from
$786$k to $917$k points in fp32 and from $1$M to $1.57$M
in bf16.

Precision affects the speedup.
Under autocast, retaining fp32 dots gives
$1.20$--$1.31$ times the eager training throughput at
$262$k and $1$M points, compared with
$1.53$--$1.70$ times using bf16 dots.
At $262$k points with fp32 inputs, the optional TF32 mode
gives a $1.45$-times speedup, compared with approximately
$1.35$ times for fp32 dots.

\begin{table}[t]
\caption{\textbf{Complete single-GPU sweep behind Table~\ref{tab:sys}}
(one job, eager PyTorch, no compilation, no checkpointing). Each cell is
median time in ms with peak memory in GB; ``OOM'' $=$ does not fit in
$95$\,GB. ``ours'' is the exact-fp32-dot tier in the fp32 block and the
bf16 tensor-core tier under autocast.}
\label{tab:sys-eager-full}
\begin{center}\small
\begin{tabular}{lcccc}
\toprule
 & \multicolumn{2}{c}{inference: ms (GB)} & \multicolumn{2}{c}{training step: ms (GB)} \\
\cmidrule(lr){2-3}\cmidrule(lr){4-5}
$N$ & eager & ours & eager & ours \\
\midrule
\multicolumn{5}{@{}l}{\emph{fp32}} \\
$262$k    & $58.4$ ($2.6$)     & $35.3$ ($2.1$)     & $171$ ($28.0$)     & $126$ ($24.0$)     \\
$1$M   & $229$ ($10.1$)     & $138$ ($8.1$)      & OOM                & OOM                \\
$2$M      & $475$ ($20.2$)     & $280$ ($16.2$)     & OOM                & OOM                \\
$4$M      & $980$ ($40.3$)     & $566$ ($32.3$)     & OOM                & OOM                \\
$8.4$M    & $1944$ ($80.5$)    & $1160$ ($64.5$)    & OOM                & OOM                \\
$12.6$M   & OOM                & OOM                & OOM                & OOM                \\
\midrule
\multicolumn{5}{@{}l}{\emph{bf16 autocast}} \\
$262$k    & $45.6$ ($2.0$)     & $29.6$ ($1.5$)     & $143$ ($20.4$)     & $93.1$ ($15.4$)    \\
$1$M   & $174$ ($7.6$)      & $113$ ($5.6$)      & $585$ ($81.3$)     & $345$ ($61.3$)     \\
$2$M      & $349$ ($15.2$)     & $226$ ($11.2$)     & OOM                & OOM                \\
$4$M      & $700$ ($30.3$)     & $456$ ($22.3$)     & OOM                & OOM                \\
$8.4$M    & $1420$ ($60.5$)    & $932$ ($44.5$)     & OOM                & OOM                \\
$12.6$M   & OOM                & $1411$ ($66.7$)    & OOM                & OOM                \\
$16.8$M   & OOM                & OOM                & OOM                & OOM                \\
\bottomrule
\end{tabular}
\end{center}
\end{table}

\paragraph{Slice count and depth.}
Table~\ref{tab:sys-scaling} varies $G$ and $L$ at
$N=262$k points.
From $G=16$ to $128$, fused bf16 peak memory remains
approximately $15.4$\,GB, while eager memory increases
from $17.9$ to $36.9$\,GB.
Over the same range, the training speedup increases
from $1.31$ to $2.18$ times.

Peak memory grows approximately linearly with depth for
both implementations, with a smaller slope for the fused
model.
At $G=32$, the tested configurations fit up to $32$ layers
in fp32 and $48$ in bf16, compared with $16$ and $32$
for eager.
Figure~\ref{fig:systems-time} shows the corresponding
step times.

\begin{table}[t]
\caption{\textbf{Structural scaling}: peak memory of a training step at
$N{=}262$k as the two configuration knobs move, each block measured in one
job. ``OOM'' $=$ does not fit in $95$\,GB. Fused peak memory changes little over the tested range of $G$; both grow linearly in depth, with different slopes.}
\label{tab:sys-scaling}
\begin{center}\small
\begin{tabular}{lrrrrc}
\toprule
 & \multicolumn{2}{c}{fp32, peak GB} & \multicolumn{2}{c}{bf16, peak GB} & bf16 step \\
\cmidrule(lr){2-3}\cmidrule(lr){4-5}
 & eager & ours & eager & ours & speedup \\
\midrule
\multicolumn{6}{@{}l}{\emph{slice count $G$} (depth $L{=}8$)} \\
$G{=}16$  & $26.0$ & $24.0$ & $17.9$ & $\mathbf{15.4}$ & $1.31\times$ \\
$G{=}32$  & $28.0$ & $24.0$ & $20.4$ & $\mathbf{15.4}$ & $1.54\times$ \\
$G{=}64$  & $32.0$ & $24.0$ & $25.4$ & $\mathbf{15.4}$ & $1.83\times$ \\
$G{=}128$ & $41.7$ & $24.0$ & $36.9$ & $\mathbf{15.4}$ & $2.18\times$ \\
\midrule
\multicolumn{6}{@{}l}{\emph{depth $L$} ($G{=}32$)} \\
$L{=}4$   & $15.0$ & $13.0$ & $10.9$ & $8.4$  & $1.52\times$ \\
$L{=}8$   & $28.0$ & $24.0$ & $20.4$ & $15.4$ & $1.51\times$ \\
$L{=}16$  & $54.1$ & $46.1$ & $39.5$ & $29.5$ & $1.42\times$ \\
$L{=}32$  & OOM    & $90.2$ & $77.6$ & $57.6$ & $1.50\times$ \\
$L{=}48$  & OOM    & OOM    & OOM    & $85.7$ & --- \\
\bottomrule
\end{tabular}
\end{center}
\end{table}

\begin{figure}[t]
\centering
\includegraphics[width=0.76\linewidth]{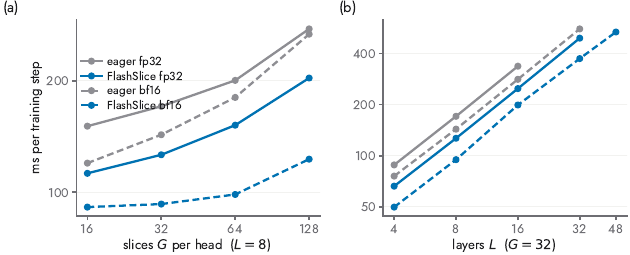}
\caption{\textbf{Step time of the sweeps of Figure~\ref{fig:systems}(b,~c).}
Median time of a training step against (a)~slice count and (b)~depth
($N{=}262$k, eager PyTorch against the fused kernels, no compilation or
checkpointing; log axes). The speedup grows with $G$, from $1.31\times$ at
$G{=}16$ to $2.18\times$ at $G{=}128$ under bf16, and is a constant factor
in depth.}
\label{fig:systems-time}
\end{figure}

\paragraph{Compiled baselines.}
Compilation substantially improves the baseline, so
Table~\ref{tab:sys-single} also compares against
\texttt{torch.compile}.
The compiled fused configuration with fast dots is
$1.24$--$1.26$ times faster per training step than the
compiled baseline in the tested settings.
These are configuration-level comparisons: the fused
rows use TF32 dots for fp32 inputs and bf16 tensor-core
dots for bf16 inputs.
The corresponding inference speedups are $1.4$--$1.5$.

\begin{table}[t]
\caption{\textbf{Single-GPU training step (ms) against compiled
baselines}, medians of $\geq 10$ iterations; each precision block is
measured in one job, and anchors reproduce across jobs within $1\%$.
``fast dot'' $=$ TF32 (fp32 inputs) or bf16 tensor-core (bf16 inputs)
tier.}
\label{tab:sys-single}
\begin{center}\small
\begin{tabular}{lccc}
\toprule
 & fp32, $262$k & bf16, $262$k & bf16, $1$M \\
\midrule
eager                        & $170.5$ & $142.9$ & $586$ \\
eager $+$ compile            & $134.8$ & $79.5$  & $300$ \\
fused (exact fp32 dots)      & $128.4$ & $118.9$ & $448$ \\
fused $+$ compile $+$ fast dot & $\mathbf{106.6}$ & $\mathbf{63.9}$ & $\mathbf{241}$ \\
\midrule
vs.\ strongest baseline      & $1.26\times$ & $1.24\times$ & $1.25\times$ \\
\bottomrule
\end{tabular}
\end{center}
\end{table}

\paragraph{Activation checkpointing.}
Table~\ref{tab:sys-ckpt} compares both implementations with
non-reentrant activation checkpointing and compilation.
The fused configurations are $1.31$--$1.34$ times faster
in bf16 and $1.39$--$1.42$ times faster with fp32 inputs,
with approximately $9\%$ lower memory use.

In these tests, both implementations train up to $4$M
points per GPU in bf16.
Thus, with checkpointing and compilation enabled, fusion
reduces runtime and memory but does not increase the
largest tested training size.

\begin{table}[t]
\caption{\textbf{Checkpointed training (ms), one job.}
\texttt{torch.compile} traverses non-reentrant activation checkpointing in this stack, halving the checkpointed baseline; the baseline uses checkpointing and compilation.}
\label{tab:sys-ckpt}
\begin{center}\small
\begin{tabular}{lcccc}
\toprule
 & bf16, $1$M & bf16, $4$M & fp32, $1$M & fp32, $2$M \\
\midrule
ckpt                     & $740$ & $2999$ & $921$ & $1874$ \\
ckpt $+$ compile         & $365$ & $1479$ & $696$ & $1432$ \\
fused ckpt $+$ compile $+$ fast dot & $\mathbf{278}$ & $\mathbf{1106}$ & $\mathbf{500}$ & $\mathbf{1006}$ \\
\midrule
vs.\ ckpt $+$ compile    & $1.31\times$ & $1.34\times$ & $1.39\times$ & $1.42\times$ \\
\bottomrule
\end{tabular}
\end{center}
\end{table}

\paragraph{Large slice counts.}
Table~\ref{tab:sys-blocked} evaluates one attention layer
using the $G$-blocked kernels.
At $N=262$k, fused bf16 training memory remains about
$2$\,GB from $G=48$ to $1024$.
At $G=512$ and $N=1$M, eager runs out of memory while
the fused layer uses $8.1$\,GB.

Runtime depends strongly on dot precision.
For the reported large-$G$ training configurations,
bf16 tensor-core modes give speedups of
$3.98$--$6.45$ times, whereas fp32 FMA modes can be
slower than eager.
The \texttt{tf32x3} mode gives speedups of up to
$1.92$ times with fp32 inputs.

The kernels are not faster for every shape.
At $D=256$, the tested blocked configuration reaches
$0.79$ times eager throughput with $4\%$ lower memory.
At the paper configuration, $G=32$, the single-tile
implementation is faster than the blocked implementation,
which requires additional statistics and recomputation.

\begin{table}[htbp]
\caption{\textbf{One layer at large slice counts, eager against the
$G$-blocked kernels} ($H{=}8$, $D{=}32$, one GH200, medians of ten;
training step unless marked). ``inputs / dots'' is the activation dtype
and the dot-precision level of Table~\ref{tab:dotmodes}; ``OOM'' $=$ does
not fit in $95$\,GB.}
\label{tab:sys-blocked}
\begin{center}\small
\setlength{\tabcolsep}{4.5pt}
\resizebox{\linewidth}{!}{%
\begin{tabular}{llrrrcrrc}
\toprule
inputs / dots & $G$ & $N$ & eager (ms) & ours (ms) & speedup & eager (GB) & ours (GB) & saved \\
\midrule
bf16 / bf16 & $256$  & $262$k & $47.5$  & $11.9$  & $3.98\times$ & $13.7$ & $2.0$ & $85\%$ \\
bf16 / bf16 & $256$  & $1$M   & $293.7$ & $45.5$  & $6.45\times$ & $54.5$ & $8.1$ & $85\%$ \\
bf16 / bf16 & $512$  & $262$k & $90.6$  & $20.0$  & $4.53\times$ & $26.7$ & $2.0$ & $92\%$ \\
bf16 / bf16 & $512$  & $1$M   & OOM     & $80.3$  & ---          & ---    & $8.1$ & ---    \\
bf16 / bf16 & $1024$ & $262$k & $180.2$ & $37.6$  & $4.79\times$ & $52.7$ & $2.1$ & $96\%$ \\
bf16 / bf16, inference & $256$ & $1$M   & $39.4$  & $13.3$ & $2.96\times$ & $17.0$ & $2.1$ & $88\%$ \\
bf16 / bf16, inference & $256$ & $4.2$M & $157.3$ & $53.2$ & $2.95\times$ & $68.0$ & $8.3$ & $88\%$ \\
bf16 / bf16 & $48$   & $262$k & $13.0$  & $6.1$   & $2.12\times$ & $3.1$  & $2.0$ & $35\%$ \\
\midrule
fp32 / tf32x3 & $256$ & $262$k & $45.4$  & $37.4$  & $1.21\times$ & $14.8$ & $2.0$ & $86\%$ \\
fp32 / tf32x3 & $256$ & $1$M   & $283.6$ & $147.7$ & $1.92\times$ & $59.0$ & $8.1$ & $86\%$ \\
fp32 / tf32   & $256$ & $262$k & $45.1$  & $40.5$  & $1.11\times$ & $14.8$ & $2.0$ & $86\%$ \\
fp32 / ieee   & $256$ & $262$k & $45.1$  & $53.2$  & $0.85\times$ & $14.8$ & $2.0$ & $86\%$ \\
fp32 / ieee   & $256$ & $1$M   & $284.2$ & $209.6$ & $1.36\times$ & $59.0$ & $8.1$ & $86\%$ \\
bf16 / tf32   & $256$ & $262$k & $47.5$  & $40.0$  & $1.19\times$ & $13.7$ & $2.0$ & $85\%$ \\
bf16 / ieee   & $256$ & $262$k & $47.4$  & $64.0$  & $0.74\times$ & $13.7$ & $2.0$ & $85\%$ \\
bf16 / ieee   & $256$ & $1$M   & $294.3$ & $253.4$ & $1.16\times$ & $54.5$ & $8.1$ & $85\%$ \\
\bottomrule
\end{tabular}}
\end{center}
\end{table}

\subsection{Multi-GPU execution and compilation}

\paragraph{Point sharding.}
We shard one sample's points across four GPUs.
In the forward pass, each GPU computes local token
numerators and weight sums, which are all-reduced before
normalization.
At the paper configuration, this reduction contains
approximately $33$\,KB per layer.
Sharded and single-GPU forward results agree within
fp32 reduction-order differences.

Table~\ref{tab:sys-shard} reports training at $16.8$M
points with checkpointing and bf16.
The fused compiled configuration takes $1152$\,ms per
step, compared with $3043$\,ms for the compiled eager
pipeline, giving approximately $2.6$ times the throughput.
Peak memory decreases from $86.6$ to $78.6$\,GB per GPU.
Step timings include parameter-gradient synchronization.

Compilation has little effect on the eager sharded
pipeline in this software stack.
The custom-operator implementation composes with
compilation and checkpointing in the tested configuration.
The measured speedup therefore reflects the full pipeline,
including its compilation behavior.

\begin{table}[t]
\caption{\textbf{One node, $4\times$GH200, $16.8$M points, bf16}, all
configurations in one job; DDP-equivalent step including the gradient
all-reduce. Compilation gives little improvement on the eager sharded pipeline; the
operator form composes.}
\label{tab:sys-shard}
\begin{center}\small
\begin{tabular}{lccc}
\toprule
 & step (ms) & Mpts/s & GB/GPU \\
\midrule
ckpt $+$ sharding                   & $3048$ & $5.5$ & $86.5$ \\
ckpt $+$ sharding $+$ compile       & $3043$ & $5.5$ & $86.6$ \\
fused ckpt $+$ compile $+$ bf16 dot & $\mathbf{1152}$ & $\mathbf{14.6}$ & $78.6$ \\
\bottomrule
\end{tabular}
\end{center}
\end{table}

\paragraph{Implementation settings.}
All compiled measurements use static per-shape graphs.
For Triton~3.0, the TF32 path uses one pipeline stage,
and tensor-core configurations with $G=16$ use four
warps to avoid compilation failures encountered during
tuning.
These settings and the kernel tile tables are included
in the implementation.

\subsection{Composition with the tiling of Transolver-3}
\label{app:tiling}
The geometry slice tiling of \citet{zhou2026transolver3} splits the points
into tiles, recomputes the slice and deslice of each tile in the backward pass
under gradient checkpointing, and sums the slice partials across tiles. We
implement this schedule in PyTorch and measure one physics-attention layer
($H{=}8$, $D{=}32$, $G{=}32$), forward and backward, on one GH200 in fp32
(Table~\ref{tab:rb-tiling}). The tiling lowers the peak memory of the eager
layer by $45\%$ and is $1.3\times$ slower. \flashslice{} alone is $3\times$
faster than the eager layer at $18\%$ less memory. The two compose: with the
fused kernels inside each tile, the layer keeps the memory of the tiling and is
$2.4$--$2.5\times$ faster than it at every size with $100$k-point tiles. Both
tiled forms match the eager layer to $7\times10^{-7}$ in the output and
$4\times10^{-6}$ in the gradients, relative.

\begin{table}[htbp]
\caption{\textbf{\flashslice{} and the tiling of Transolver-3.} One physics-attention layer
($H{=}8$, $D{=}32$, $G{=}32$), forward and backward on one GH200, fp32 with TF32 disabled:
time (ms) / peak memory (GB). ``\flashslice{} in tiling'' runs the fused kernels inside
each tile.}
\label{tab:rb-tiling}
\begin{center}\footnotesize
\setlength{\tabcolsep}{5pt}
\begin{tabular}{lccc}
\toprule
 & $N{=}262$k & $N{=}1.0$M & $N{=}4.2$M \\
\midrule
eager                              & $32$ / $2.8$ & $131$ / $11.1$ & $525$ / $44.1$ \\
\flashslice{}                      & $11$ / $2.3$ & $43$ / $9.1$   & $174$ / $36.1$ \\
tiling, 100k-point tiles           & $44$ / $1.6$ & $174$ / $6.1$  & $698$ / $24.1$ \\
tiling, 20k-point tiles            & $54$ / $1.6$ & $204$ / $6.1$  & $872$ / $24.1$ \\
\flashslice{} in tiling, 100k      & $18$ / $1.6$ & $72$ / $6.1$   & $284$ / $24.1$ \\
\flashslice{} in tiling, 20k       & $47$ / $1.6$ & $166$ / $6.1$  & $668$ / $24.1$ \\
\bottomrule
\end{tabular}
\end{center}
\end{table}

\section{Additional controlled experiments}
\label{app:controls}
This appendix collects further controlled experiments.
They use the surface benchmarks and the protocol of
Appendix~\ref{app:protocol}, with seed $42$ unless stated otherwise, and each
is compared with a \baselinev{} trained under the same conditions. The second attention $\times$ tying grid and the cost of the LinearNO layer are in
Appendix~\ref{app:linearno}, and the composition with the tiling of
Transolver-3 is in Appendix~\ref{app:tiling}.

\paragraph{Coupling frequency.} Table~\ref{tab:rb-coupling} separates the two
changes that \sliceonce{} makes at once. \kcoupling{} keeps the attention sublayer
at layer $5$ for $K{=}1$, at layers $3$ and $7$ for $K{=}2$, and at every other
layer for $K{=}4$.

\begin{table}[htbp]
\caption{\textbf{How often must the coupling be repeated?} \kcoupling{} keeps the
attention sublayer in $K$ of the $L{=}8$ layers (evenly spaced) and the pointwise MLP
in every layer; \sliceonce{}$+$\textsc{Dec} adds four pointwise MLP blocks after the
single deslice. Best validation relative $L^{1}$ and ratio to a \baselinev{} trained
under the same conditions; $^{\dagger}$from Table~\ref{tab:main}.}
\label{tab:rb-coupling}
\begin{center}\footnotesize
\setlength{\tabcolsep}{5pt}
\begin{tabular}{lccccc}
\toprule
 & coupling & \multicolumn{2}{c}{SHIFT-SUV surface} & \multicolumn{2}{c}{DrivAerML surface} \\
\cmidrule(lr){3-4}\cmidrule(lr){5-6}
variant & rounds & error & ratio & error & ratio \\
\midrule
\mlponly{}$^{\dagger}$ ($K{=}0$)        & 0 & $0.252$  & $1.61\times$  & $0.501$  & $5.46\times$ \\
\kcoupling{}, $K{=}1$                   & 1 & $0.1598$ & $1.023\times$ & $0.1263$ & $1.377\times$ \\
\kcoupling{}, $K{=}2$                   & 2 & $0.1576$ & $1.009\times$ & $0.1031$ & $1.124\times$ \\
\kcoupling{}, $K{=}4$                   & 4 & $0.1579$ & $1.012\times$ & $0.0933$ & $1.018\times$ \\
\baselinev{} ($K{=}8$)                  & 8 & $0.1561$ & $1$           & $0.0917$ & $1$ \\
\midrule
\sliceonce{}$^{\dagger}$                & 1 & $0.449$  & $2.87\times$  & $0.579$  & $6.31\times$ \\
\sliceonce{}$+$\textsc{Dec}             & 1 & $0.2001$ & $1.28\times$  & $0.2769$ & $3.02\times$ \\
\bottomrule
\end{tabular}
\end{center}
\end{table}

\paragraph{Token mixers.} Table~\ref{tab:rb-mixers} replaces the token
self-attention by three other mixers of the $G$ tokens of each head: a learned
constant $G{\times}G$ core, initialized at the identity; a token-mixing MLP
across the slot axis; and a bidirectional selective state-space scan across
the slots, in the manner of \citet{tiwari2025lamo}. Slice, deslice and the
point stream are those of \baselinev{}. All three lie within $2\%$ of the
attention, and the constant core recovers what the identity core of
\notoken{} loses on this benchmark.

\begin{table}[htbp]
\caption{\textbf{The token mixer, replaced.} DrivAerML surface, the benchmark where removing
the attention costs most. Every row keeps slice, deslice and the point stream of
\baselinev{} and changes only what acts on the $G$ tokens of each head. Best validation
relative $L^{1}$; ratio of the seed mean to \baselinev{}.}
\label{tab:rb-mixers}
\begin{center}\footnotesize
\setlength{\tabcolsep}{5pt}
\begin{tabular}{llcccc}
\toprule
token mixer & tokens interact through & seed 42 & seed 43 & mean & ratio \\
\midrule
softmax attention (\baselinev{})        & content-dependent $A(z)$ & $0.0917$ & -- & -- & $1$ \\
identity (\notoken{})                   & none ($A{=}I$)            & $0.0961$ & -- & -- & $1.047\times$ \\
constant $G{\times}G$ core              & learned constant $A$      & $0.0899$ & $0.0907$ & $0.0903$ & $0.984\times$ \\
token-mixing MLP                        & MLP over the slot axis    & $0.0883$ & $0.0927$ & $0.0905$ & $0.987\times$ \\
bidirectional selective SSM             & input-dependent scan      & $0.0925$ & $0.0941$ & $0.0933$ & $1.017\times$ \\
\bottomrule
\end{tabular}
\end{center}
\end{table}

\paragraph{Depth.} Table~\ref{tab:rb-depth} tests
Corollary~\ref{cor:mlponly-depth}. A four-fold change of depth moves the error
of \mlponly{} by $6.5\%$, and its gap to the coupled model stays between
$1.56\times$ and $1.61\times$.

\begin{table}[htbp]
\caption{\textbf{Depth does not repair \mlponly{}} (Corollary~\ref{cor:mlponly-depth}).
SHIFT-SUV surface, best validation relative $L^{1}$. $^{\dagger}$From
Table~\ref{tab:main}; $^{\ddagger}$early-stopped at $49$k of $50$k steps.}
\label{tab:rb-depth}
\begin{center}\footnotesize
\setlength{\tabcolsep}{6pt}
\begin{tabular}{lccc}
\toprule
depth $L$ & \baselinev{} & \mlponly{} & ratio \\
\midrule
4  & $0.1677$ & $0.2642$ & $1.58\times$ \\
8  & $0.1561$ & $0.2519^{\dagger}$ & $1.61\times$ \\
16 & $0.1581^{\ddagger}$ & $0.2471$ & $1.56\times$ \\
\bottomrule
\end{tabular}
\end{center}
\end{table}

\end{document}